\documentclass{article}
\usepackage{ijcai26}

\usepackage{times}
\usepackage{soul}
\usepackage{url}
\usepackage[hidelinks]{hyperref}
\usepackage[utf8]{inputenc}
\usepackage[small]{caption}
\usepackage{graphicx}
\usepackage{amsmath}
\usepackage{amsthm}
\usepackage{booktabs}
\usepackage{algorithm}
\usepackage{algorithmic}
\usepackage[switch]{lineno}
\usepackage{amssymb} 
\usepackage{mathtools}
\usepackage{graphicx}
\usepackage{float}
\usepackage{placeins}
\usepackage{booktabs}
\usepackage{multirow}
\usepackage{tabularx}
\usepackage[table]{xcolor}

\usepackage{siunitx}
\usepackage{graphicx} 
\usepackage{subcaption} 
\usepackage{amsmath}    
\usepackage{amsfonts}   
\usepackage{wrapfig}
\usepackage{graphicx}
\usepackage{xcolor}
\definecolor{lightblue}{RGB}{0,114,188}  
\definecolor{myred}{RGB}{162,37,38} 

\newtheorem{theorem}{Theorem}
\newtheorem{lemma}{Lemma}
\newtheorem{definition}{Definition}
\newtheorem{proposition}{Proposition}

\title{VGAS: Value-Guided Action-Chunk Selection for Few-Shot Vision-Language-Action Adaptation}

\author{
Changhua Xu
\and
En Yu
\and
Junyu Xuan
\and
Jie Lu 
\\
\affiliations
Australian Artificial Intelligence
Institute (AAII), 
University of Technology Sydney, Australia
\\
\emails
changhua.xu@student.uts.edu.au;
\{en.yu-1, junyu.xuan, jie.lu\}@uts.edu.au
}

\begin{document}

\maketitle

\begin{abstract}
Vision--Language--Action (VLA) models bridge multimodal reasoning with physical control, but adapting them to new tasks with scarce demonstrations remains unreliable.
While fine-tuned VLA policies often produce semantically plausible trajectories, failures often arise from unresolved geometric ambiguities, where near-miss actions lead to divergent execution outcomes under limited supervision. 
We study few-shot VLA adaptation from a \emph{generation--selection} perspective and propose a novel framework \textbf{VGAS} (\textbf{V}alue-\textbf{G}uided \textbf{A}ction-chunk \textbf{S}election). It performs inference-time best-of-$N$ selection to identify action chunks that are both semantically faithful and geometrically precise. Specifically, \textbf{VGAS} employs a finetuned VLA as a high-recall proposal generator and introduces the \textrm{Q-Chunk-Former}, a geometrically grounded Transformer critic to resolve fine-grained geometric ambiguities.  In addition, we propose \textit{Explicit Geometric Regularization} (\texttt{EGR}), which shapes a discriminative value landscape to preserve action ranking resolution among near-miss candidates while mitigating value instability under scarce supervision. Experiments and theoretical analysis demonstrate that \textbf{VGAS} consistently improves success rates and robustness under limited demonstrations and distribution shifts. Our code is available at \url{https://github.com/Jyugo-15/VGAS}.

\end{abstract}

\section{Introduction}

Vision-Language-Action (VLA) models have emerged as a transformative paradigm for embodied AI, bridging multimodal reasoning with physical control~\cite{brohan2022rt,zitkovich2023rt,he2026finegrained}. By pretraining on vast robotic datasets, these generalist policies learn to map complex visual observations and linguistic instructions directly into executable actions~\cite{black2024pi0,kim2024openvla,team2024octo,intelligence2025pi05}. However, the reliability of VLAs in downstream applications remains heavily bottlenecked by the prevailing Supervised Fine-Tuning (SFT) paradigm~\cite{kim2025fine,zhang2025pure}. SFT-based adaptation demands a high volume of expert demonstrations to bridge the gap between generalist priors and task-specific requirements. This is often challenging in real-world environments where high-quality robotic data collection is costly, unscalable, and prone to out-of-distribution (OOD) uncertainties~\cite{dass2022pato,xin2024programmatic,sapkota2025vision,yu2026generalized}. Consequently, under data-scarce regimes, VLA policies often exhibit brittle performance, failing to generalize across even minor distribution shifts~\cite{liu2025can,li2025simplevla,yu2025learning,tan2505interactive,guo2025improving}.

\begin{figure}[!t]
  \centering
  \includegraphics[width=0.95\linewidth]{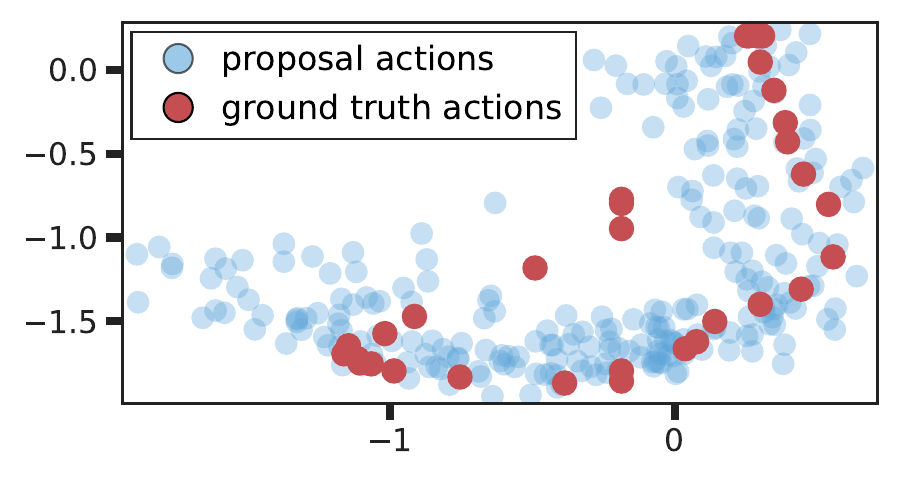}
  \caption{Illustration of near-miss action distribution under 5-shot VLA fine-tuning.}
  \label{fig:action_dis}
\end{figure}

To study this data-scarce regime, we simulate a realistic few-shot adaptation setting by fine-tuning a pretrained VLA policy with only five demonstrations per task, and evaluating it on held-out demonstrations with different initial spatial configurations (e.g., object positions). As shown in Figure~\ref{fig:action_dis}, the fine-tuned policy generally preserves the task semantics, producing actions toward the correct object. However, due to the limited state--action coverage, even slight variations in the initial spatial configuration can make the policy's action predictions less concentrated, yielding dispersed near-miss predictions around the ground-truth actions. Although semantically reasonable, these predictions are often geometrically imprecise and can cause failures such as inaccurate grasps, end-effector pose deviations, or joint-angle overshoot~\cite{kumar2022pre,zhao2023learning}. This suggests that few-shot VLA adaptation is primarily limited by geometric precision under sparse supervision, rather than semantic understanding alone. Motivated by this, we reformulate adaptation as a value-guided selection problem: instead of requiring a generative policy to jointly acquire semantic reasoning and fine-grained geometric control end-to-end, we decouple adaptation into high-recall proposal generation and high-precision value-based selection, prioritizing candidates with the highest likelihood of long-horizon success.

Offline Reinforcement Learning (ORL) provides a natural framework for this selection objective, as it learns an outcome-aware critic that maps long-horizon success into a scalar value signal~\cite{sutton1998reinforcement}—perfectly suited for ranking candidate proposals~\cite{ghasemipour2021emaq,janner2022planning,luodreamfuser}. However, applying existing ORL methodologies to modern VLA policies exposes two fundamental limitations: 
\textit{1) Structural and observational mismatch:} Standard RL assumes per-step atomic actions, whereas modern VLAs output temporally extended action chunks, inducing an SMDP structure that complicates value learning and temporal credit assignment~\cite{sutton1999between}. Moreover, most offline RL is evaluated with compact, near-Markovian state inputs; in contrast, VLA relies on high-dimensional vision--language observations with geometric grounding, making value estimation substantially harder and comparatively under-explored~\cite{fu2020d4rl,lu2022challenges}. \textit{2) Ranking resolution vs.\ conservatism:} Offline RL often controls extrapolation to low-support actions via conservative objectives or behavior-regularized extraction~\cite{kumar2020conservative,kostrikov2021offline}. In few-shot sparse-reward settings, however, such regularization can compress value gaps among proposal-supported near-miss candidates, yielding low-contrast gradients and weakening inference-time Best-of-$N$ selection~\cite{lyu2022mildly}. These limitations raise two major research questions: \textbf{\emph{RQ1.}} \emph{What critic architecture can robustly ground high-dimensional VLA observations into precise value estimates for temporally extended action chunks?} and \textbf{\emph{RQ2.}} \emph{How can a value function be trained under scarce demonstrations to maintain high ranking resolution among near-miss action chunks?}

To address these questions, we propose \textbf{VGAS} (\textbf{V}alue-\textbf{G}uided \textbf{A}ction-chunk \textbf{S}election) for VLA adaptation via generation--selection decoupling.  For \textbf{\emph{RQ1}}, we introduce \textrm{Q-Chunk-Former}, a geometrically grounded critic architecture built on a Transformer backbone. By leveraging the Transformer's sequence modeling capability, our design naturally captures temporal dependencies within action chunks. Crucially, our architecture preserves fine-grained, token-level features, allowing attention to explicitly focus on geometric cues that are critical for precise value estimation. 
For \textbf{\emph{RQ2}}, we propose a hybrid offline RL objective that \emph{anchors} temporal consistency via a proposal-constrained Bellman backup, augmented by Explicit Geometric Regularization (\texttt{EGR}). Unlike traditional conservative methods that indiscriminately penalize out-of-distribution actions, \texttt{EGR} injects dense geometric supervision, shaping the value landscape into a smooth funnel anchored at expert demonstrations. This allows the critic to maintain high ranking resolution among near-miss candidates even under scarce supervision. Our contributions are summarized as follows:
\begin{itemize}
    \item We reformulate few-shot VLA adaptation as a value-guided selection problem and propose the novel \textbf{VGAS} method, shifting the paradigm from likelihood-based generation to outcome-aware ranking.
    \item  We propose \textrm{Q-Chunk-Former}, which enables precise geometric grounding for action-chunk evaluation. We introduce \texttt{EGR}, a regularization technique that injects dense geometric priors into offline RL to maintain high ranking resolution under data-scarce regimes.
    \item We provide theoretical guarantees for the convergence of our chunk-level value operator and demonstrate through extensive experiments on the LIBERO benchmark that VGAS consistently outperforms SFT and standard ORL baselines, particularly in terms of success rate and robustness under distribution shifts.
\end{itemize}
\begin{figure*}[!htbp]
  \centering
  \includegraphics[width=.9\textwidth]{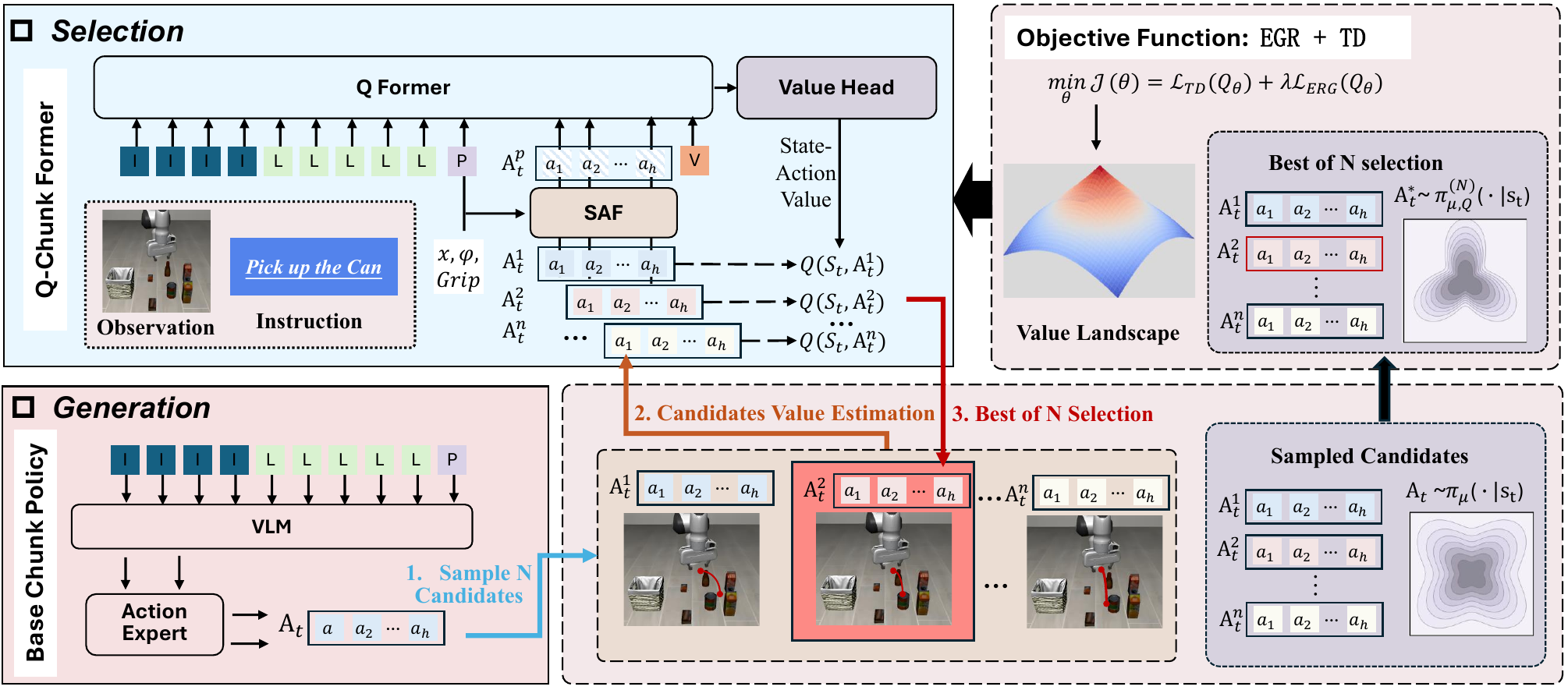}
\caption{The overall framework of \textbf{VGAS}.
\textbf{Generation:} A fine-tuned VLA policy proposes $N$ candidate action chunks from multimodal inputs.
\textbf{Selection:} \textrm{Q-Chunk-Former} learns a scoring function $Q$ via the \texttt{EGR}+\texttt{TD} objective.
Best-of-$N$ selection defines the induced policy $\pi_{\mu,Q}^{(N)}$ by maximizing over a discriminative value landscape shaped by EGR, prioritizing expert-aligned candidates and thereby mitigating geometric drift.}
  \label{fig:framework}
\end{figure*}

\section{Preliminary}
\paragraph{Offline Reinforcement Learning.}
We consider a Markov Decision Process (MDP) defined by $\mathcal{M}=(\mathcal{S},\mathcal{A},\mathcal{P},r,\rho,\gamma)$, where $s\in\mathcal{S}$, $a\in\mathcal{A}$, $\mathcal{P}(s'\mid s,a)$ is the transition kernel, $r(s,a)\in\mathbb{R}$ is the reward function, $\rho_0\in\Delta(\mathcal{S})$ is the initial-state distribution, and $\gamma\in[0,1)$ is the discount factor. The objective is to learn a policy $\pi$ maximizing the expected discounted return $J(\pi)=\mathbb{E}_{\pi}\!\left[\sum_{t\ge 0}\gamma^t r(s_t,a_t)\right]$.

Value-based methods estimate the optimal action-value function $Q^\star$ as the fixed point of the Bellman optimality operator~\cite{fujimoto2019off,levine2020offline}:
\begin{equation}
\small
(\mathcal{T}^\star Q)(s,a)
= r(s,a) + \gamma\,\mathbb{E}_{s'\sim \mathcal{P}(\cdot\mid s,a)}
\left[\max_{a'\in\mathcal{A}} Q(s',a')\right].
\end{equation}

In offline RL, the agent learns from a fixed dataset $\mathcal{D}$ collected by a behavior policy $\pi_\beta$,
without additional environment interaction~\cite{bacchiocchi2024online}.
A key challenge is extrapolation error~\cite{zhang2023adaptive}: the maximization over $a'$ may select actions outside the data support,
leading to overestimation and instability.
Many offline RL methods ~\cite{shin2023guide} mitigate this via conservative regularization~\cite{kumar2020conservative},
which discourages high values on out-of-distribution actions.

\paragraph{Action Chunking in VLAs.}
Modern VLAs condition on multimodal inputs, which we denote by the state
$s_t=(I_t,L_t,p_t)$, comprising visual tokens $I_t$, language instruction $L_t$,
and robot proprioception $p_t$~\cite{zitkovich2023rt}.
With a slight abuse of notation, we treat this policy input as the MDP state.
Instead of per-step control, these models often output a temporally extended \emph{action chunk}
with horizon $h$:
\begin{equation}
\mathrm{A}_t := (a_t,\ldots,a_{t+h-1}) \in \mathcal{A}^h,
\qquad
\mathrm{A}_t \sim \pi_{\mu}(\cdot \mid s_t).
\label{eq:chunk_def}
\end{equation}


\paragraph{Few-shot VLA Adaptation Objective.}
Given a few-shot expert dataset $\mathcal{D}=\{\tau_i\}_{i=1}^{K}$, we first obtain a task-aligned base chunk policy $\pi_\mu(\mathrm{A}\mid s)$, for example via SFT on $\mathcal D$. In this regime, SFT can capture task semantics yet may struggle to resolve fine-grained geometric ambiguities. We therefore treat $\pi_\mu$ as a high-recall proposal distribution and assume non-trivial local support around each demonstrated chunk:
\begin{equation}
\pi_\mu\!\left(\{\mathrm A:\|\mathrm A-\mathrm A_t\|_2\le \varepsilon\}\mid s_t\right)\ge p_0,
 \forall (s_t,\mathrm A_t)\in\mathcal D,
\label{eq:local_support}
\end{equation}
for some $\varepsilon>0$ and $p_0>0$. Intuitively, 
$p_0$ 
 ensures the proposal covers valid near-miss candidates.
Our goal is to improve upon $\pi_\mu$ by learning an offline-adapted chunk policy within a
proposal-constrained class $\Pi_\mu$:
\begin{equation}
\pi^\star := \arg\max_{\pi \in \Pi_\mu} J(\pi),
\label{eq:fewshot_objective}
\end{equation}
where $\Pi_\mu$ denotes policies supported by (or centered at) the proposal $\pi_\mu$. Under chunked execution, the return is
\begin{equation}
    J(\pi):= \mathbb{E}_{\tau \sim \pi}\!\left[ \sum_{k=0}^{\infty} (\gamma^h)^k \, R_h(s_{t_k}, \mathrm{A}_{t_k}) \right],
\label{eq:chunk_return}
\end{equation}
where $t_k = kh$ denotes the start time of the $k$-th chunk, $\mathrm{A}_{t_k} \sim \pi(\cdot \mid s_{t_k})$ is the action chunk, and $R_h(s_{t_k}, \mathrm{A}_{t_k}) = \sum_{j=0}^{h-1} \gamma^{j}\, r\!\left(s_{t_k+j}, a_{t_k+j}\right)$ is the discounted cumulative reward over the chunk.

\section{Methodology}

\paragraph{Framework Overview.} 
\textbf{VGAS} reformulates few-shot VLA adaptation as a generate-then-select  process. As illustrated in Figure~\ref{fig:framework}, our pipeline decouples the policy into two components: a \emph{high-recall generator} and a \emph{high-precision critic}. First, we utilize a supervised fine-tuned (SFT) VLA model as the base policy $\pi_\mu(\mathrm{A}_t \mid s_t)$ to provide a proposal distribution covering plausible action chunks. At inference time, we sample $N$ candidates $\{\mathrm{A}_t^{(i)}\}_{i=1}^N \sim \pi_\mu(\cdot \mid s_t)$ and employ a learned critic $Q_\theta$ to execute Best-of-$N$ selection, $\mathrm{A}_t^\star = \arg\max_i Q_\theta(s_t, \mathrm{A}_t^{(i)})$. This strategy approximates policy improvement within the support of $\pi_\mu$, prioritizing geometric precision without requiring online exploration.

To realize this selection mechanism effectively, VGAS addresses two core challenges: \emph{representation} and \emph{optimization}. First, regarding critic representation (Sec.\ref{subsec:q_arch}), we introduce \textrm{Q-Chunk-Former}, a Transformer-based architecture tailored for VLA inputs. To prevent high-dimensional visual tokens from overwhelming physical cues, we design a State-Action Fusion (SAF) module that explicitly grounds action chunks in proprioceptive states before multimodal integration.
Second, for critic optimization (Sec.\ref{subsubsec:td_chunked_emax}), we propose a hybrid learning objective that combines temporal and spatial supervision. We stabilize offline training using a \emph{Proposal-Constrained Chunked Expected-Max} backup, which enforces temporal consistency across action chunks. We further augment this with \emph{Explicit Geometric Regularization} (\texttt{EGR}), a dense supervision signal that directly shapes the value landscape based on geometric proximity to expert demonstrations, enabling the critic to reliably distinguish near-miss actions from failures even under sparse task rewards.

\subsection{Q-Chunk-Former}
\label{subsec:q_arch}
To optimize the few-shot adaptation objective in Eq.~\eqref{eq:fewshot_objective}, the critic network must accurately estimate the long-horizon value of a temporally extended action chunk $\mathrm{A}_t$ given state $s_t$. This imposes two key requirements,
{(i) Chunk-level evaluation with temporal structure:}
the critic must assign a single long-horizon value to an entire action chunk $\mathrm{A}_t$ for Best-of-$N$ selection
(Eq.~\eqref{eq:bestofn_select}),  while preserving the within-chunk temporal ordering of actions;
{(ii) Multimodal fusion without geometric collapse:} unlike classical critics operating on compact state vectors,
VLA conditioning involves heterogeneous inputs such as vision, language, and proprioception,
where naive compression can discard geometry-critical cues needed for feasibility-aware value estimation.

Motivated by token-level multimodal modeling~\cite{marafioti2025smolvlm}, we introduce a Transformer-based critic.
A naive design treats all modalities as a single concatenated token sequence:
\begin{equation}
Q_\theta(s_t,\mathrm{A}_t)
=\textrm{Transformer}_\theta\!\left([I_t,\,L_t,\,p_t,\,\mathrm{A}_t]\right),
\label{eq:naive_concat}
\end{equation}
where $[\cdot]$ denotes concatenation. However, in practice, the self-attention mechanism tends to be dominated by the abundant visual and linguistic tokens, so the single proprioceptive token $p_t$ receives insufficient attention. This is detrimental because value estimation for manipulation requires joint reasoning over the external world context (from $I_t$ and $L_t$) and the internal robot state (embodiment and configuration encoded by $p_t$). If $p_t$ is under-utilized, the critic becomes less sensitive to geometric feasibility, weakening feasibility-aware value estimation.

To address this, we introduce a lightweight State-Action Fusion (\textrm{SAF}) module that conditions the raw chunk $\mathrm{A}_t$ on $p_t$ prior to mixing with high-dimensional perceptions. The \textrm{SAF} module produces proprioception-grounded action tokens $\mathrm{A}_t^{p}$,
\begin{equation}
\mathrm{A}_t^{p} = \textrm{SAF}(\mathrm{A_t},p_t) = \mathrm{W}_{\text{fuse}} \left( [\mathrm{W}_a \mathrm{A}_t \parallel \mathrm{W}_p p_t] \right),
\label{eq:saf}
\end{equation}
where $\mathrm{W}_a$ and $\mathrm{W}_p$ are learnable projections mapping inputs to a shared latent space and $\mathrm{W}_{\text{fuse}}$ aggregates the concatenated features. This design enforces a high-fidelity interaction between action tokens and the proprioceptive state, ensuring that feasibility cues are embedded directly into the chunk representation. The grounded action tokens $\mathrm{A}_t^{p}$ are then integrated with the frozen perceptual tokens\footnote{Perceptual tokens $I_t$ and $L_t$ are extracted from the pre-trained VLM encoder of the base policy to ensure feature alignment and computational efficiency.} and a learnable $\mathrm{[VALUE]}$ token $\mathrm{v}$ via a Transformer decoder, denoted as the \textrm{Q-Former} (\textrm{QF}). Finally, a \textrm{Value Head} (\textrm{VH}) maps the output embedding of $\mathrm{v}$ to a scalar Q-value:
\begin{equation}
Q_\theta(s_t,\mathrm{A}_t)
=
\textrm{VH}\!\left(
\textrm{QF}\!\left(
I_t,\,L_t,\,p_t,\,\mathrm{A}_t^{p},\,\mathrm{v}
\right)\right).
\label{eq:q_former_arch}
\end{equation}

In summary, \textrm{Q-Chunk-Former} comprises \textrm{SAF}, \textrm{QF}, and \textrm{VH} modules. This architecture mitigates attention imbalance and ensures that value estimation is strictly grounded in geometric reality, providing a reliable signal for selection.

\subsection{Optimization Objective}
\label{subsec:q_learning}

Our goal is to learn a critic that supports \emph{value-guided selection} over a fixed proposal distribution.
We assume access to a task-adapted base chunk policy $\pi_\mu(\mathrm{A}\mid s)$
that generates semantically plausible action chunks.
At inference time, we draw $N$ i.i.d.\ candidates from $\pi_\mu$ and select the best one according to a scoring rule $Q$:
\begin{equation}
\mathrm{A}_Q^\star(s)
~:=~
\arg\max_{i\in[N]} Q\!\left(s,\mathrm{A}^{(i)}\right).
\label{eq:bestofn_select}
\end{equation}

Although the maximization in Eq.~\eqref{eq:bestofn_select} is deterministic conditioned on the sampled set,
proposal sampling induces a stochastic \emph{selection policy}. We denote this induced policy by
$\pi_{\mu,Q}^{(N)}(\mathrm{A}\mid s)$, defined as the distribution of $\mathrm{A}_Q^\star(s)$.

To learn Q from a fixed offline dataset $\mathcal{D}$, we train the critic as a proxy for the long-horizon return of the induced policy $\pi_{\mu,Q}^{(N)}$ in the chunk-induced SMDP. A reliable critic in the few-shot regime must satisfy two
coupled desiderata: (i) \textbf{temporal consistency}, i.e., Bellman-style alignment with the expected \emph{max} return under best-of-$N$ selection over $\pi_\mu$ samples; and (ii) \textbf{spatial consistency}, i.e., preserving fine-grained geometric ranking among near-miss proposals. We instantiate these principles with the following hybrid objective:
\begin{equation}
\min_{\theta}\; \mathcal{J}(\theta)
~=~
\mathcal{L}_{\texttt{TD}}(Q_{\theta})
~+~
\lambda\,\mathcal{L}_{\texttt{EGR}}(Q_{\theta}).
\label{eq:total_objective}
\end{equation}

\paragraph{Temporal Consistency.}
\label{subsubsec:td_chunked_emax}
The primary goal of $\mathcal{L}_{\texttt{TD}}$ is to align critic learning with our inference-time execution,
i.e., to learn the value function induced by Best-of-$N$ selection, $\pi_{\mu,Q}^{(N)}$.
This alignment requires two ingredients.
First, since execution commits to a length-$h$ action chunk, the critic must evaluate the return of an entire chunk.
We therefore adopt the chunk-level \texttt{TD} formulation from Q-Chunking~\cite{li2025reinforcement} and learn
$Q(s_t,\mathrm A_t)$ that conditions on the full action chunk.

Second, we must ensure that the \emph{Bellman backup} matches the same Best-of-$N$ rule used at inference.
To this end, inspired by Expected-Max Q-learning (EMaQ)~\cite{ghasemipour2021emaq}, the key insight of EMaQ is to construct a Bellman backup that replaces the standard expectation under a policy with an expected maximization over a set of sampled proposals. By targeting this maximum, the objective ensures that the critic trained with this operator is guaranteed to converge to the value function of the corresponding Best-of-$N$ selection policy  $\pi_{\mu,Q}^{(N)}$. Thus, we define the proposal-constrained \emph{Chunked Expected--Max} backup operator $\mathcal{T}_\mu^{N}$ as
\begin{equation}
\centering
\begin{aligned}
&(\mathcal{T}_\mu^{N} Q)(s,\mathrm{A}):= R_h(s,\mathrm{A}) 
 + \\& \gamma^h\mathbb{E}_{s'\sim \mathcal{P}_h(\cdot\mid s,\mathrm{A})}
\Big[
\mathbb{E}_{\mathrm{A}'_{1:N}\overset{\mathrm{i.i.d.}}{\sim}\pi_\mu(\cdot\mid s')}
\big[
\max_{i\in[N]} Q(s',\mathrm{A}'_i)
\big]
\Big],
\end{aligned}
\label{eq:chunked_emax_operator}
\end{equation}
where $R_h$ is the discounted cumulative reward over the chunk and $\mathcal{P}_h$ is the $h$-step transition kernel.
The inner maximization corresponds exactly to selecting $\mathrm{A}_Q^\star(s')$ from Eq.~\eqref{eq:bestofn_select},
thereby enforcing strict train--test consistency.

This operator formulation offers rigorous theoretical guarantees for offline adaptation, which we summarize below.

\begin{proposition}[Chunked Expected--Max in tabular SMDPs]
\label{prop:chunked_emax_properties}
In the tabular chunk-induced SMDP, assume bounded rewards and $\gamma^h\in(0,1)$.
Then $\mathcal{T}_\mu^{N}$ in Eq.~\eqref{eq:chunked_emax_operator} is a $\gamma^h$-contraction under $\|\cdot\|_\infty$
and has a unique fixed point $Q_\mu^{N}$.
Let $\pi_\mu^{(N)} := \pi_{\mu,Q_\mu^{N}}^{(N)}$ be the induced Best-of-$N$ policy. Then
$Q_\mu^{N} = Q^{\pi_\mu^{(N)}}$.
Monotonicity in $N$ and the $N\!\to\!\infty$ limit are given in
Prop. 3 and Thm. 2 (App. B).
\end{proposition}

Proposition~\ref{prop:chunked_emax_properties} shows that the proposed backup is well-defined in the tabular induced SMDP and that its unique fixed point corresponds exactly to the value function of the induced Best-of-$N$ selection policy. A detailed proof of these properties in the tabular SMDP setting is provided in Appendix \ref{app:convergence_chunked_emax}.

Based on this operator, we instantiate the final temporal-difference loss with a standard target network $Q_{\bar{\theta}}$ for stability:
\begin{equation}
\begin{aligned}
&\mathcal{L}_{\texttt{TD}}(\theta)
=
\mathbb{E}_{(s_t,\mathrm{A}_t,s_{t+h})\sim\mathcal{D}}
\Big[
\big(
Q_{\theta}(s_t,\mathrm{A}_t) - y_t
\big)^2
\Big], \\
&\text{where }
y_t=R_h(s_t,\mathrm{A}_t)
+\gamma^h \max_{i\in[N]} Q_{\bar{\theta}}(s_{t+h}, \mathrm A_{t+h}^{\prime i}).
\end{aligned}
\label{eq:td_loss_final}
\end{equation}

Here, the proposals $\{\mathrm A_{t+h}^{\prime i}\}_{i=1}^{N}$ are sampled from the frozen proposal distribution $\pi_\mu(\cdot\mid s_{t+h})$ at the next decision state. This loss provides a stable, proposal-constrained temporal anchor for our critic, paving the way for the spatial regularization described next.

\paragraph{Spatial Consistency.} 

\label{subsubsec:egr}

While the proposal-constrained \texttt{TD} loss provides a stable temporal anchor, offline critic learning remains vulnerable to extrapolation error, where OOD (off-demo) candidates deviating from the training data distribution receive spuriously high values. Standard conservative methods, such as CQL~\cite{kumar2020conservative}, mitigate this by indiscriminately suppressing values for all low-support actions. However, in the few-shot regime, this uniform penalty is overly aggressive, because it compresses the value dynamic range among proposal-supported candidates and reduces the fine-grained ranking resolution needed to distinguish plausible near-miss candidates from catastrophic failures. This value collapse directly undermines inference-time Best-of-$N$ selection.

To address this, we introduce \emph{Explicit Geometric Regularization (\texttt{EGR})}. Instead of uniformly suppressing off-demo actions, \texttt{EGR} serves as a structural regularizer. During training, we regularize the critic with off-demo candidates sampled from \(\rho(\cdot \mid s_t)\), a proposal-centered mixture detailed in Appendix ~\ref{sec:exp_setup}. At inference time, Best-of-$N$ selects only among proposal samples from $\pi_\mu(\cdot\mid s_t)$. Accordingly, \texttt{EGR} targets two desiderata: (\textbf{i}) \emph{TD-Anchored Calibration:} anchor the value scale to the \texttt{TD} target $y_t$ to \emph{mitigate} overestimation on off-demo candidates; (\textbf{ii}) \emph{Geometric Discriminability:} within the inference-time proposal set, preserve a graded preference for geometric proximity, providing a smooth signal to separate recoverable near-misses from divergence.

We formalize \texttt{EGR} as the following weighted combination of \emph{Anchoring} and \emph{Ranking} losses.
\begin{equation}
\mathcal{L}_{\mathtt{EGR}}(\theta)
= \mathcal{L}_{\mathtt{anchor}}(\theta)
+\eta\,\mathcal{L}_{\mathtt{rank}}(\theta),
\label{eq:egr_total}
\end{equation}
where $\eta\ge 0$ balances the \emph{absolute scale} (Anchoring) and the \emph{local ordering} (Ranking).

\textit{a. Geometric anchoring ($\mathcal{L}_{\mathrm{anchor}}$).}
For any off-demo action chunk
$\hat{\mathrm A}_t \sim \rho(\cdot\mid s_t)$, we define
\begin{equation}
\mathcal{Y}(s_t,\hat{\mathrm A}_t)
\coloneqq
\operatorname{sg}(y_t)
-\beta||\hat{\mathrm A}_t - \mathrm A_t \rVert_{\mathcal W}^2,
\label{eq:ref_surface}
\end{equation}
where $\operatorname{sg}(\cdot)$ stops gradients and $y_t$ is the \texttt{TD} target from Eq.~\eqref{eq:td_loss_final}.
Crucially, we do not posit Euclidean distance as a global ground-truth metric; rather, we employ this surface as a
\emph{structural inductive bias} to shape the value landscape in regions where task supervision is absent. In our setting, we employ a weighted metric $\|\hat{\mathrm A}_t - \mathrm A_t \|_{\mathcal W}^2$ to prioritize critical kinematic dimensions (e.g., end-effector position) within the normalized control space. This serves as a robust local proxy for geometric proximity: under smooth dynamics, small weighted action deviations tend to induce small short-horizon trajectory deviations. (See Appendix \ref{weight_detail} for details).

To satisfy \texttt{TD}-anchored calibration, we align the critic’s estimates on candidate chunks with the reference surface:
\begin{equation}
\mathcal{L}_{\mathrm{anchor}}(\theta; s_t,\mathrm A_t)
=
\mathbb{E}_{\hat{\mathrm A}_t\sim\rho(\cdot\mid s_t)}
\Big[
\ell\big(Q_\theta(s_t,\hat{\mathrm A}_t),\,\mathcal{Y}(s_t,\hat{\mathrm A}_t)\big)
\Big],
\label{eq:egr_anchor}
\end{equation}
where $\ell(\cdot,\cdot)$ is the squared error.

\textit{b. Geometric ranking ($\mathcal{L}_{\mathrm{rank}}$).}
While anchoring constrains the absolute scale, it does not guarantee robust \emph{local} discrimination. To enforce
geometric discriminability, we introduce a pairwise ranking loss. For any pair
$(\hat{\mathrm{A}}_t^{i}, \hat{\mathrm{A}}_t^{j}) \sim \rho(\cdot\mid s_t)$, define squared distances to the expert:
\begin{equation}
d_i \coloneqq \left\lVert \hat{\mathrm{A}}_t^{i}-\mathrm{A}_t \right\rVert_{\mathcal W}^2,
\qquad
d_j \coloneqq \left\lVert \hat{\mathrm{A}}_t^{j}-\mathrm{A}_t \right\rVert_{\mathcal W}^2.
\label{eq:egr_pair_dist}
\end{equation}
By Eq.~\eqref{eq:ref_surface}, the reference differential value satisfies
$\mathcal{Y}(s_t,\hat{\mathrm{A}}_t^{i})-\mathcal{Y}(s_t,\hat{\mathrm{A}}_t^{j})=\beta(d_j-d_i)$.
We thus encourage the critic to preserve the same relative differences among candidate chunks:
\begin{equation}
\begin{aligned}
    \mathcal{L}_{\mathrm{rank}}(\theta; s_t,\mathrm A_t) =
\mathbb{E}_{\hat{\mathrm{A}}_t^{i},\,\hat{\mathrm{A}}_t^{j}\sim \rho(\cdot\mid s_t)}
\Big[
\ell\!\Big(
Q_{\theta}(s_t,\hat{\mathrm{A}}_t^{i}) \\ 
-Q_{\theta}(s_t,\hat{\mathrm{A}}_t^{j}), \beta(d_j-d_i)\Big)
\Big].
\end{aligned}
\label{eq:egr_rank}
\end{equation}

\paragraph{The Closed Loop of Spatio-Temporal Consistency.}
\label{subsubsec:synergy}
The \texttt{TD} and \texttt{EGR} form a mutually reinforcing loop for safe and effective value learning. The Expected-Max TD objective provides a foundational safety layer by inherently operating within the support of the proposal distribution $\pi_\mu$. However, this implicit constraint alone can still be vulnerable to selecting outlier candidates that are accidentally overestimated. \texttt{EGR} reinforces this safety by explicitly shaping the OOD value landscape into a geometric funnel. This structure actively biases the Best-of-$N$ maximization toward candidates that remain close to expert behavior, providing a tighter and more reliable bound on the \texttt{TD} target. This enhanced safety guarantee is formalized by the following proposition:
\begin{proposition}[Best-of-$N$ bound under an \texttt{EGR} anchoring envelope]
\label{prop:egr_bestofn_bound}
Fix a demonstration pair $(s_t,\mathrm A_t)\sim\mathcal D$ and a candidate distribution
$\pi_\mu(\cdot\mid s_t)$. Define the \texttt{EGR} reference surface
$\mathcal Y(s_t,\mathrm A'_t)=\operatorname{sg}(y_t)-\beta\|\mathrm A'_t-\mathrm A_t\|_{\mathcal W}^2$
(Eq.~\eqref{eq:ref_surface}) and the \emph{anchoring residual}
$\delta_\theta(s_t,\mathrm A'_t)\coloneqq Q_\theta(s_t,\mathrm A'_t)-\mathcal Y(s_t,\mathrm A'_t)$.
Assume it is uniformly upper-bounded on the candidate set:
\[
\sup_{\mathrm A'_t\in \mathrm{supp}(\pi_{\mu}(\cdot\mid s_t))}\delta_\theta(s_t,\mathrm A'_t)\le \varepsilon.
\]
Draw $N$ candidates $\{\mathrm A_t^{\prime i}\}_{i=1}^N\sim\rho(\cdot\mid s_t)$. Then
\begin{equation}
\max_{i\in[N]} Q_\theta(s_t,\mathrm A_t^{\prime i})
\le
\operatorname{sg}(y_t)+\varepsilon.
\label{eq:bestofn_bound_main}
\end{equation}
See Appendix \ref{app:egr_bestofn_bound} for the full statement and proof.
\end{proposition}

\emph{TD Anchors and Calibrates EGR.}
Conversely, the \texttt{TD} objective provides an essential grounding signal for \texttt{EGR}. The \texttt{TD} target $y_t$ sets the absolute scale for the \texttt{EGR} reference surface, preventing the geometric shaping from degenerating into an uncalibrated ranking function. This synergy allows the \texttt{TD}-\texttt{EGR} loop to progressively correct value estimates: \texttt{TD} provides the return-aware scale, while \texttt{EGR} provides the fine-grained geometric structure, jointly enabling robust ranking.

\begin{table*}[t]
\centering
\setlength{\tabcolsep}{4pt}
\renewcommand{\arraystretch}{1.15}

\begin{tabular}{@{} l l
S[table-format=2.1] S[table-format=1.0]
S[table-format=2.1] S[table-format=1.0]
S[table-format=2.1] S[table-format=1.0]
S[table-format=2.1] S[table-format=1.0]
S[table-format=2.1] S[table-format=2.1] @{}}
\toprule
& & \multicolumn{2}{c}{LIBERO-Spatial}
  & \multicolumn{2}{c}{LIBERO-Object}
  & \multicolumn{2}{c}{LIBERO-Goal}
  & \multicolumn{2}{c}{LIBERO-Long}
  & \multicolumn{2}{c}{Average} \\
\cmidrule(lr){3-4}\cmidrule(lr){5-6}\cmidrule(lr){7-8}\cmidrule(lr){9-10}\cmidrule(lr){11-12}
Type & Method
& {SR ($\uparrow$)} & {Rank ($\downarrow$)}
& {SR ($\uparrow$)} & {Rank ($\downarrow$)}
& {SR ($\uparrow$)} & {Rank ($\downarrow$)}
& {SR ($\uparrow$)} & {Rank ($\downarrow$)}
& {SR ($\uparrow$)} & {Rank ($\downarrow$)} \\
\midrule
BC-Only & SmolVLA
& 46.0 & {4}    
& 45.8 & {4}    
& 52.0 & {3}    
& 15.5 & {4}    
& 39.8 & {4}    
\\
\midrule
\multirow{4}{*}{BC + RL}
& QC-M
& {46.0} & {4} & {42.1} & {5} & {49.2} & {5} & {14.4} & {5} & {37.9} & {5} \\
& QC-M+CQL
& {47.8} & {2} & {50.3} & {2} & {54.2} & {2} & {16.5} & {3} & {42.2} & {2} \\
& QC-T+CQL
& {47.7} & {3} & {50.1} & {3} & {51.9} & {4} & {{17.5}} & {2} & {41.8} & {3} \\
\cmidrule(l){2-12}
\rowcolor{gray!20}
& \textbf{VGAS (Ours)} & \textbf{56.2} & \textbf{1} & \textbf{59.0} & \textbf{1} & \textbf{60.8} & \textbf{1} & \textbf{20.0} & \textbf{1} & \textbf{49.0} & \textbf{1} \\
\bottomrule
\end{tabular}
\caption{Results on LIBERO (5-shot per task). SR denotes success rate ($\%$).}
\label{tab:libero_cross}
\end{table*}

\section{Experiment}
\label{sec:experiments}

In our experiments, we aim to answer the following research questions: \emph{RQ1: Does \textbf{VGAS} benefit from a Transformer-based chunk critic (Q-Chunk-Former) for modeling fine-grained geometric dependencies under multimodal inputs?} \emph{RQ2: Does Explicit Geometric Regularization (\texttt{EGR}) improve value calibration and ranking over state--action-chunk candidates for Best-of-$N$ selection?}

\subsection{Experiment Settings}
\label{sec:exp_settings}

\paragraph{Benchmark and Architecture.}
We evaluate \textbf{VGAS} on the widely used simulation benchmark, LIBERO~\cite{liu2023libero}. LIBERO is a lifelong learning benchmark focused on language-guided manipulation tasks across diverse object types, task specifications, and environments scenarios. Specifically, it includes 4 suites: Goal, Spatial, Object, and Long. Each suite is designed to evaluate a specific aspect of object manipulation and containing 10 distinct tasks.
We use SmolVLA-0.5B~\cite{shukor2025smolvla} as the base chunk policy.
We initialize \textrm{Q-Chunk-Former} with the first two decoder layers of the pre-trained SmolVLM2~\cite{marafioti2025smolvlm}, keeping the critic's token space aligned with the policy.
Comprehensive architectural details are provided in Appendix \ref{sec:impl_details}.

\paragraph{Baselines.}
We compare \textbf{VGAS} against baselines covering VLA fine-tuning and offline value-based improvement. \textbf{BC-only} fine-tunes SmolVLA with behavior cloning. \textbf{CQL}~\cite{kumar2020conservative} serves as a representative conservative offline RL objective and is widely used in VLA settings~\cite{song2025hume,chebotar2023q,huang2025co,nakamoto2025steering,chen2025conrft}. \textbf{QC-M} follows Q-Chunking~\cite{li2025reinforcement}, training an MLP critic with standard TD learning over action chunks. \textbf{QC-M+CQL} augments \textbf{QC-M} with the CQL regularizer. Finally, \textbf{QC-T+CQL} retains the same objective but replaces the MLP with our \textrm{Q-Chunk-Former}, enabling a controlled comparison of critic architectures under identical conservative constraints. Implementation details are provided in Appendix \ref{sec:baseline}.

\subsection{Main Results}
\label{sec:main result}

Table \ref{tab:libero_cross} reports LIBERO success rates. The unregularized Q-Chunking baseline (\textbf{QC-M}) underperforms BC ($37.9\%$ vs.\ $39.8\%$), reflecting a maximization issue in offline RL. With sparse demonstrations, the critic can overestimate slightly off-demonstration near-miss chunks due to function approximation error, and Best-of-$N$ selection then preferentially picks these overvalued outliers, inducing compounding drift as rollouts enter poorly covered state--action regions~\cite{kumar2020conservative,mark2024policy}.

Adding CQL alleviates this failure, improving QC-M from $37.9\%$ to $42.2\%$ (\textbf{QC-M+CQL}), but the gain over BC remains modest. Replacing the MLP with our Transformer backbone yields a similar result ($41.8\%$ for \textbf{QC-T+CQL}), suggesting that indiscriminate conservative suppression compresses value gaps among proposal-supported candidates, leaving Best-of-$N$ with insufficient ranking resolution.

Finally, \textbf{VGAS} achieves $49.0\%$, outperforming \textbf{QC-T+CQL} by a clear margin. Since both methods share the same \textrm{Q-Chunk-Former} backbone, this improvement is primarily attributed to Explicit Geometric Regularization (\texttt{EGR}). See Sec.~\ref{sec:static_analysis} for further comparative analysis.

\FloatBarrier
\begin{figure*}[!htbp]
  \centering
  \includegraphics[width=.95\textwidth]{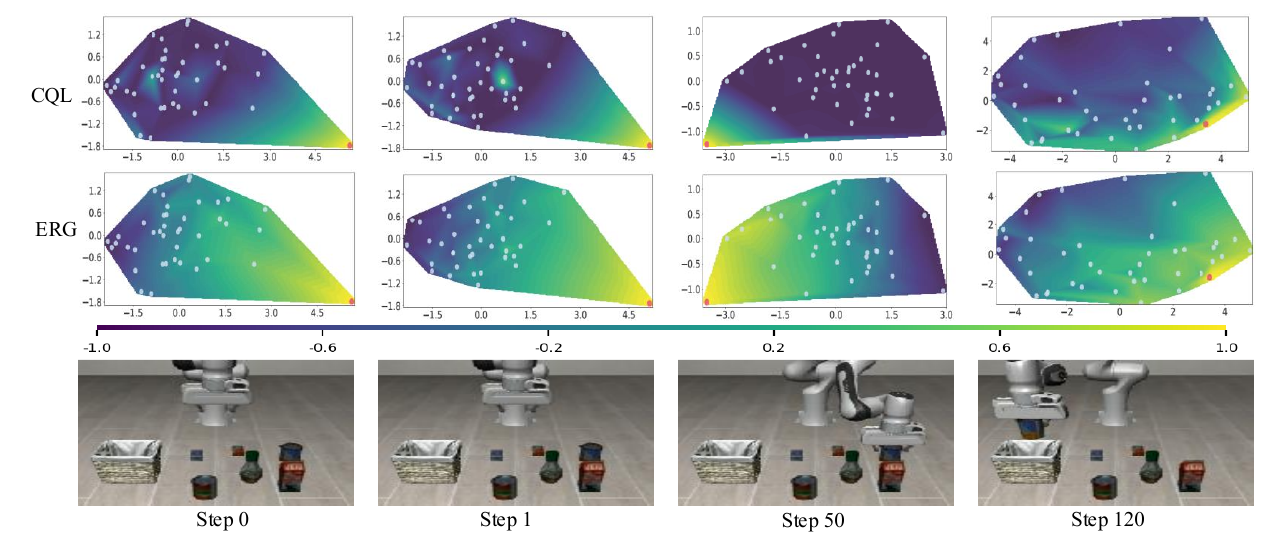}
  \caption{Visualization of the Proposal-Candidate Value Landscape: CQL vs. EGR (Ours)}
  \label{fig:value_map_comparison}
\end{figure*}

\subsection{Ablation Studies}
\label{sec:ablation studys}
Table~\ref{table:ablation} isolates the contribution of each component. The dominant gain arises from \emph{Explicit Geometric Regularization} ($\mathcal{L}_{\texttt{EGR}}$). Removing $\mathcal{L}_{\texttt{EGR}}$ leads to a sharp drop in success rate from $49.0\%$ to $36.4\%$, effectively collapsing performance back to the unregularized QC-M baseline ($37.9\%$). This regression suggests that, without geometric shaping, the critic is susceptible to a canonical offline RL failure mode: it assigns spuriously high values to off-demo candidates, and the subsequent Best-of-$N$ maximization amplifies these overestimations. These results indicate that \texttt{EGR} plays a dual role: beyond suppressing erroneous high values on off-demo actions, it explicitly structures the local value landscape around demonstrations, preserving the ranking resolution required for reliable Best-of-$N$ selection. In addition, ablating the temporal consistency term ($\mathcal{L}_{\texttt{TD}}$) reduces performance to $45.5\%$, showing that while \texttt{EGR} provides strong spatial guidance, a TD-based anchor remains necessary to stabilize long-horizon value estimates.

\begin{table}[h]
\centering
\begin{tabular}{@{}l|ccccc@{}}
\toprule
Variants & Spatial & Object & Goal & LONG & Avg \\
\midrule
w/o \textrm{SAF} & 51.8& 54.2& 56.8& 17.4& 45.1 \\
w/o $\mathrm{QCF}$ & {53.4} & 55& 55.9& {16.8} & 45.3 \\
\midrule
w/o $\mathcal{L}_{\texttt{EGR}}$ & 42.6 & 42.3 & 48.2& 12.5 &36.4\\
w/o $\mathcal{L}_{\texttt{TD}}$ & 55.2 & 54.5&56.0& 16.4 &45.5\\
\midrule
\textbf{VGAS} & 56.2 & {59.0} & 60.8 & 20.0 & 49 \\
\bottomrule
\end{tabular}
\caption{Ablation study on LIBERO (success rates ($\%$)).}
\label{table:ablation}
\end{table}

Architectural choices prove equally critical.
We denote our Transformer-based critic as $\mathrm{QCF}$ (\textrm{Q-Chunk-Former}).
Replacing $\mathrm{QCF}$ with a standard MLP backbone (w/o $\mathrm{QCF}$) yields $45.3\%$,
validating the necessity of the Transformer's attention mechanism for modeling complex multimodal dependencies.
Notably, ablating the State-Action Fusion module (w/o \textrm{SAF}) further lowers performance to $45.1\%$.
This suggests that without explicit grounding, the critic struggles to resolve fine-grained geometric ambiguities
amidst high-dimensional visual features.

\begin{figure}[t]
  \centering
  \begin{subfigure}[t]{0.5\linewidth}
    \centering
    \includegraphics[width=\linewidth]{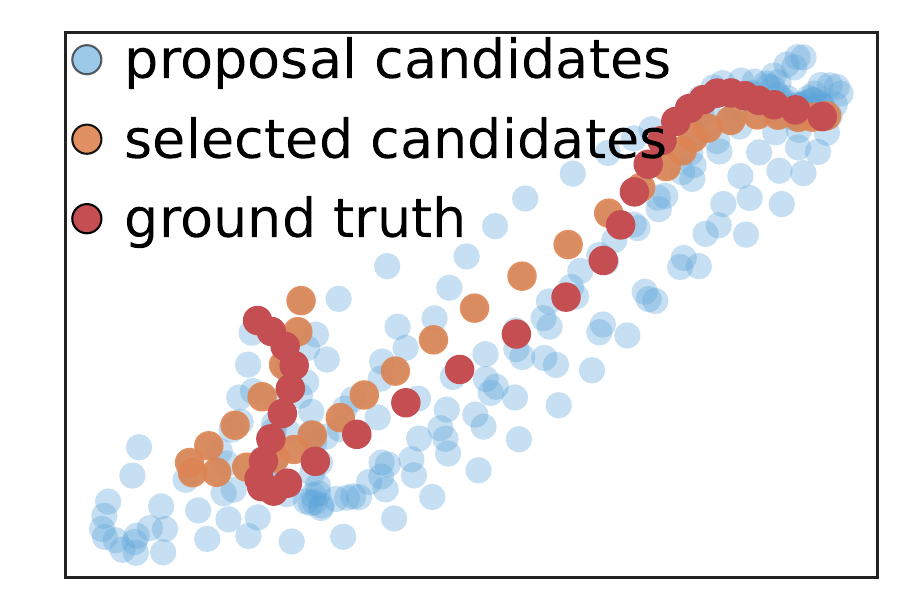}
    \caption{Projection on X-Z Plane}
    \label{fig:decision_profile_a}
  \end{subfigure}\hfill
  \begin{subfigure}[t]{0.5\linewidth}
    \centering
    \includegraphics[width=\linewidth]{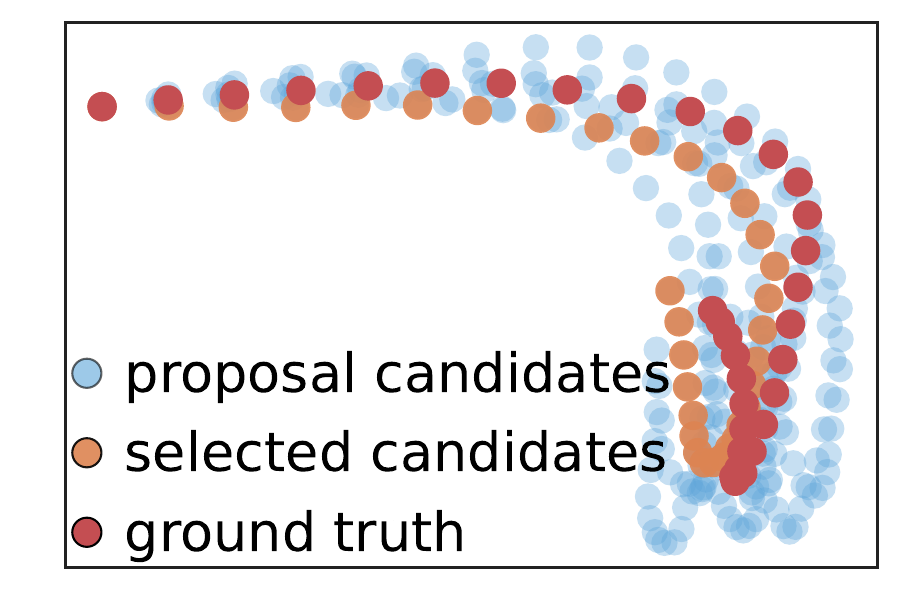}
    \caption{Projection on Y-Z Plane}
    \label{fig:decision_profile_b}
  \end{subfigure}
  \caption{Multi-view Spatial Rollouts of Action Chunks and \textbf{VGAS} Selection. Trajectories are reconstructed via temporal integration in orthogonal views. \textbf{VGAS} identifies the trajectory aligning with the expert across 3D space.}
  \label{fig:decision_profile}
\end{figure}

\subsection{Visualization Analysis}
\label{sec:static_analysis}

Figure~\ref{fig:value_map_comparison} visualizes the landscape of estimated action-chunk values projected onto a 2D plane, with $Q$-values normalized to $[-1, 1]$ for comparison. As shown in the top row, standard conservative regularization (CQL) results in a collapsed value landscape. It indiscriminately suppresses the Q-values of all proposal candidates (white dots) to a uniformly low level. Consequently, the critic loses the resolution to differentiate recoverable near-miss proposals from failures, rendering the selection process ineffective.
In contrast, the bottom row demonstrates that VGAS successfully restores fine-grained ranking resolution. Instead of uniform suppression, the value signal exhibits a graded geometric preference that decays smoothly as candidates deviate from the expert actions (red dot). This structure ensures that the critic can meaningfully rank candidates based on their physical proximity to the optimal solution. The detailed analysis is in Appendix \ref{app:viz_analysis}.

To assess execution precision, we visualize the spatial rollouts induced by action chunks on held-out task instances. Although these instances remain semantically aligned with the training demonstrations, they contain subtle yet consequential spatial variations. The SFT baseline (blue) exhibits pronounced dispersion, producing a “cloud” of candidates that frequently drifts away from the target. This indicates poor geometric generalization under few-shot supervision: the policy tends to memorize demonstration-specific trajectories rather than adapt its execution to instance-level spatial configurations. In contrast, \textbf{VGAS} acts as a geometric stabilizer. By enforcing structural consistency rather than exact path memorization, the learned critic suppresses this variance and selects the candidate (orange) that best matches the current spatial arrangement. Consequently, \textbf{VGAS} corrects execution drift induced by rigid imitation in the base policy. Additional experiments on other VLA baselines, inference-budget sensitivity, demonstration-budget scaling, and hyperparameter sensitivity are provided in Appendix \ref{app:additional_exp}.

\section{Conclusion and Limitations}
We propose \textbf{VGAS}, which reformulates few-shot VLA adaptation as value-guided selection by decoupling high-recall proposal generation from a high-precision geometric critic. Two components jointly address the near-miss failure mode: \textrm{Q-Chunk-Former} grounds multimodal observations into chunk-level value estimates, and \texttt{EGR} prevents value-landscape collapse to preserve ranking resolution. On LIBERO, \textbf{VGAS} consistently outperforms SFT and offline RL baselines. Despite these improvements, two practical limitations remain. First, inference-time Best-of-$N$ selection incurs computational latency, limiting applicability in high-frequency control. Second, extending \textbf{VGAS} to real-world platforms remains an important direction for future work.

\section*{Acknowledgments}
The work was supported by the Australian Research Council (ARC) under Laureate project FL190100149.
\bibliographystyle{named}
\bibliography{ijcai26}

\newpage
\appendix
\onecolumn

\section{Related Work}
\label{app:relatedWorks}
\subsection{Vision-Language-Action Models.}
The intersection of computer vision and robotic control has been advanced by Vision-Language-Action (VLA) models, which endow high-capacity Vision-Language Models (VLMs) with actuation capabilities to map multimodal inputs (visual observations and natural language instructions) to executable robot actions~\cite{zitkovich2023rt,yang2026walking}. Representative architectures such as RT-2~\cite{zitkovich2023rt}, Octo~\cite{team2024octo}, OpenVLA~\cite{kim2024openvla}, and related works~\cite{brohan2022rt,black2024pi0,intelligence2025pi05} demonstrate strong generalization across diverse tasks. A common paradigm is large-scale pre-training followed by supervised fine-tuning (SFT) on robotic demonstrations, which grounds semantic knowledge into physical control.

Despite this progress, pure imitation remains limited for efficient adaptation. SFT typically requires substantial expert coverage and can degrade sharply in few-shot regimes. Moreover, many VLAs generate actions via open-loop chunking~\cite{zhao2023learning}, without an intrinsic mechanism to evaluate or rank candidate chunks by physical fidelity. This motivates adapting pretrained VLAs beyond static imitation by introducing explicit value-based evaluation for action-chunk selection.

\subsection{Value-based Offline Reinforcement Learning.}
Reinforcement learning (RL)~\cite{levine2020offline,duan2025bandwidth,yang2026walking} seeks policies that maximize long-horizon return~\cite{kong2024efficient}. In the offline setting, value functions can be learned from static datasets without additional interaction, making critic-centric methods attractive when online rollouts~\cite{bacchiocchi2024online} are costly or unsafe. To prevent extrapolation errors, Conservative offline RL methods \cite{kumar2020conservative} mitigate extrapolation error by suppressing values of out-of-distribution (OOD) actions~\cite{zhang2023adaptive,yang2025adapting,peng2025distributional}, while implicit approaches like IQL~\cite{kostrikov2021offline} tend to avoid OOD actions via behavior-regularized policy extraction. 

\textit{Sampling-based Maximization (Best-of-$N$).}
In continuous action spaces, directly computing $\max_a Q(s,a)$ is often intractable. A widely used surrogate is \emph{proposal-constrained maximization}: sample $N$ candidate actions from a behavior-aligned proposal (or a learned generative model) and select the argmax under a critic. BCQ~\cite{fujimoto2019off} is a canonical example, using a state-conditioned generative model to produce in-distribution candidates and choosing the highest-valued action via a learned Q-function. EMaQ~\cite{ghasemipour2021emaq} formalizes this principle via an \emph{Expected-Max} backup operator that interpolates between evaluation and maximization through the number of samples $N$, providing an operator-level view of Best-of-$N$ style policy improvement. More recently, generative planners such as Diffuser~\cite{janner2022planning} and value-guided diffusion methods~\cite{luodreamfuser} similarly combine sampling with value-based guidance or reranking over action sequences or trajectories, sharing the core idea of improving decision quality through learned value signals.

\subsection{RL for VLA Models.}
Integrating RL into VLA models aims to combine semantic reasoning with return-driven optimization. A core challenge is that modern VLAs often execute \emph{action chunks} (temporally extended sequences) rather than atomic actions, which shifts both credit assignment and value estimation from the single-step regime to the chunk level. Moreover, unlike traditional RL that operates on compact state vectors, VLAs condition on high-dimensional multimodal inputs (vision, language, and proprioception), making stable value learning sensitive to cross-modal fusion and representation balance.

\textit{Action Chunk Value learning.}
Recent works extend value learning to action chunking by treating a short-horizon action sequence as a single decision unit~\cite{li2025reinforcement}, showing that chunk-level Q-functions can improve temporal consistency. This line of work also introduces unbiased $n$-step backup targets to stabilize and accelerate TD learning under temporally extended actions. While promising in low-dimensional state settings, scaling chunk-level critics to high-dimensional multimodal inputs (vision, language, and proprioception) remains challenging.

\textit{Online Fine-tuning.}
A line of work fine-tunes VLAs via online RL or interactive learning~\cite{li2025simplevla,tan2505interactive,mark2024policy,guo2025improving}. These methods typically do not learn an explicit Q-function; instead, they update the policy using on-policy algorithms such as PPO~\cite{schulman2017proximal} or GRPO~\cite{shao2024deepseekmath}, where policy gradients are estimated from trajectory-level advantage signals. While effective, online fine-tuning can incur substantial interaction costs and safety concerns, making it less practical for rapid few-shot adaptation. Moreover, trajectory-level supervision can provide relatively coarse credit assignment, which may be insufficient to correct fine-grained geometric errors in long-horizon manipulation.

\textit{Offline Value Learning.}
Alternatively, approaches such as Q-Transformer~\cite{chebotar2023q} learn Q-functions from offline data to enable end-to-end control. More closely related to our setting, V-GPS~\cite{nakamoto2025steering} and Hume~\cite{song2025hume} also learn value functions from offline datasets and make decisions by scoring and selecting among candidate actions. Similarly, ConRFT~\cite{chen2025conrft} improves VLA policies via offline RL. However, most of these methods are primarily formulated for single-step action evaluation and thus do not directly align with chunk-based VLA policies that execute temporally-extended action sequences. Moreover, these approaches~\cite{chebotar2023q,nakamoto2025steering,song2025hume,yu2020multi,ijcai2025p883} typically compress multimodal observations into a single latent vector for value prediction, which can discard fine-grained cues needed for accurate value estimation and precise action ranking.

\section{Theoretical Analysis}
\label{sec:appendix_theory}

In this appendix, we provide proofs for the convergence properties of the
\emph{Proposal-Constrained Chunked Expected-Max} operator introduced in
Sec.~\ref{subsubsec:td_chunked_emax}. Our analysis adapts the operator-level
argument of EMaQ~\cite{ghasemipour2021emaq} to the induced semi-Markov decision
process (SMDP) formed by length-$h$ action chunks.

\subsection{Convergence of the Chunked Expected-Max Operator}
\label{app:convergence_chunked_emax}

\paragraph{Induced SMDP and bounded function space.}
Let $\mathrm{A}=(a_0,\dots,a_{h-1})\in\mathcal{A}^h$ denote an action chunk.
We consider the induced SMDP with (i) the chunked discounted return
\begin{equation}
R_h(s,\mathrm{A})
:= \mathbb{E}\!\left[\sum_{k=0}^{h-1}\gamma^{k}\, r(s_k,a_k)\,\middle|\, s_0=s,\mathrm{A}\right],
\label{eq:chunk_return_def}
\end{equation}
(ii) the $h$-step transition kernel $\mathcal{P}^{h}(\cdot\mid s,\mathrm{A})$,
and (iii) the effective discount factor $\gamma^h$.
Assume bounded one-step rewards $|r(s,a)|\le r_{\max}$ and $\gamma\in(0,1)$. Then
\begin{equation}
|R_h(s,\mathrm{A})|
\le \sum_{k=0}^{h-1}\gamma^{k}r_{\max}
= \frac{r_{\max}(1-\gamma^{h})}{1-\gamma}
=: R_{\max}.
\label{eq:chunk_return_bound}
\end{equation}

Let $\mathcal{B}$ denote the space of bounded real-valued functions over chunk state-action pairs:
\[
\mathcal{B} := \{ Q : \mathcal{S}\times\mathcal{A}^h \to \mathbb{R} \mid \|Q\|_\infty < \infty \},
\qquad
\|Q\|_\infty := \sup_{(s,\mathrm{A})} |Q(s,\mathrm{A})|.
\]

(We assume $Q$ is bounded and measurable; the analysis extends to continuous $\mathcal{A}^h$
(equivalently $\mathcal{B}=L_\infty(\mathcal S\times\mathcal A^h)$ under $\|\cdot\|_\infty$).)

\begin{definition}[Chunked Expected-Max Operator]
\label{def:chunked_emax_operator}
Fix a proposal distribution $\pi_\mu(\mathrm{A}\mid s)$ over action chunks and an integer $N\ge 1$.
Define $\mathcal{T}_\mu^N:\mathcal{B}\to\mathcal{B}$ by
\begin{equation}
(\mathcal{T}_{\mu}^{N} Q)(s,\mathrm{A})
:=
R_h(s,\mathrm{A})
+\gamma^h\,\mathbb{E}_{s' \sim \mathcal{P}^{h}(\cdot\mid s,\mathrm{A})}
\Big[
\mathbb{E}_{\mathrm{A}'_{1:N}\sim \pi_{\mu}(\cdot\mid s')}
\big[\max_{i\in[N]} Q(s',\mathrm{A}'_i)\big]
\Big].
\label{eq:app_chunked_emax_operator}
\end{equation}
\end{definition}

\begin{lemma}[Well-definedness]
\label{lemma:well_defined}
Assume $|r(s,a)|\le r_{\max}$ and $\gamma\in(0,1)$, hence $|R_h(s,\mathrm{A})|\le R_{\max}$ as in
\eqref{eq:chunk_return_bound}. Then for any $Q\in\mathcal{B}$,
\[
\|\mathcal{T}_\mu^N Q\|_\infty \le R_{\max} + \gamma^h \|Q\|_\infty.
\]
In particular, $\mathcal{T}_\mu^N Q\in\mathcal{B}$.
\end{lemma}

\begin{proof}
For any $Q \in \mathcal{B}$ and $(s,\mathrm{A})$, let $s' \sim \mathcal{P}^{h}(\cdot\mid s,\mathrm{A})$. We have:
\[
|(\mathcal{T}_\mu^N Q)(s,\mathrm{A})|
\le |R_h(s,\mathrm{A})| + \gamma^h \mathbb{E}_{s'}\mathbb{E}_{\mathrm{A}'_{1:N}}
\big[\max_{i\in[N]} |Q(s',\mathrm{A}'_i)|\big]
\le R_{\max} + \gamma^h \|Q\|_\infty,
\]
since $\max_i |Q(s',\mathrm{A}'_i)|\le \|Q\|_\infty$. Taking $\sup_{(s,\mathrm{A})}$ yields the claim.
\end{proof}

\begin{lemma}[Non-expansiveness of expected max]
\label{lemma:max_ineq}
Let $u,v$ be bounded functions and let $X_1,\dots,X_N$ be random variables.
Then
\begin{equation}
\Big|
\mathbb{E}\big[\max_{i\in[N]} u(X_i)\big]
-
\mathbb{E}\big[\max_{i\in[N]} v(X_i)\big]
\Big|
\;\le\;
\|u-v\|_\infty.
\label{eq:nonexp_expected_max}
\end{equation}
\end{lemma}

\begin{proof}
Using $|\mathbb{E}[Z]|\le \mathbb{E}[|Z|]$,
\[
\Big|\mathbb{E}[\max_i u(X_i)-\max_i v(X_i)]\Big|
\le \mathbb{E}\Big[\,|\max_i u(X_i)-\max_i v(X_i)|\,\Big].
\]
For any realizations $\{x_i\}$,
\[
|\max_i u(x_i)-\max_i v(x_i)|
\le \max_i |u(x_i)-v(x_i)|
\le \|u-v\|_\infty.
\]
Taking expectation completes the proof.
\end{proof}

\begin{theorem}[$\gamma^h$-Contraction]
\label{thm:contraction}
For any $N\ge 1$, the operator $\mathcal{T}_\mu^N$ is a $\gamma^h$-contraction under $\|\cdot\|_\infty$:
\[
\|\mathcal{T}_\mu^N Q_1 - \mathcal{T}_\mu^N Q_2\|_\infty
\;\le\;
\gamma^h \|Q_1-Q_2\|_\infty,
\qquad \forall\, Q_1,Q_2\in\mathcal{B}.
\]
Consequently, $\mathcal{T}_\mu^N$ admits a unique fixed point $Q_\mu^N\in\mathcal{B}$ and
$(\mathcal{T}_\mu^N)^k Q_0 \to Q_\mu^N$ for any $Q_0\in\mathcal{B}$.
\end{theorem}

\begin{proof}
Fix $(s,\mathrm{A})$ and let $Q_1,Q_2\in\mathcal{B}$. The reward terms cancel:
\begin{align*}
|(\mathcal{T}_\mu^N Q_1)(s,\mathrm{A}) - (\mathcal{T}_\mu^N Q_2)(s,\mathrm{A})|
&= \gamma^h \Big|
\mathbb{E}_{s'} \Big[
g_N(Q_1)(s') - g_N(Q_2)(s')
\Big]\Big| \\
&\le \gamma^h \mathbb{E}_{s'} \Big| g_N(Q_1)(s') - g_N(Q_2)(s')\Big|,
\end{align*}
where
\[
g_N(Q)(s')
:= \mathbb{E}_{\mathrm{A}'_{1:N}\sim \pi_{\mu}(\cdot\mid s')}
\big[\max_{i\in[N]} Q(s',\mathrm{A}'_i)\big].
\]
Applying Lemma~\ref{lemma:max_ineq} to $u(\cdot)=Q_1(s',\cdot)$ and $v(\cdot)=Q_2(s',\cdot)$ yields
$|g_N(Q_1)(s') - g_N(Q_2)(s')| \le \|Q_1-Q_2\|_\infty$ for all $s'$.
Thus
\[
|(\mathcal{T}_\mu^N Q_1)(s,\mathrm{A}) - (\mathcal{T}_\mu^N Q_2)(s,\mathrm{A})|
\le \gamma^h \|Q_1-Q_2\|_\infty.
\]
Taking $\sup_{(s,\mathrm{A})}$ gives the contraction inequality.
Since $\gamma^h<1$ and Lemma~\ref{lemma:well_defined} ensures $\mathcal{T}_\mu^N:\mathcal{B}\to\mathcal{B}$,
Banach's fixed-point theorem implies existence/uniqueness and convergence.
\end{proof}

\begin{proposition}[Monotonicity in $N$]
\label{prop:monotonicity}
Let $Q_\mu^N$ denote the unique fixed point of $\mathcal T_\mu^N$.
For any integers $N>M\ge 1$, we have
\[
Q_\mu^N(s,\mathrm A)\ \ge\ Q_\mu^M(s,\mathrm A),
\qquad \forall s\in\mathcal S,\ \forall \mathrm A\in \mathrm{supp}(\pi_\mu(\cdot\mid s)).
\]
\end{proposition}

\begin{proof}
For any bounded $Q$ and any $(s,\mathrm A)$, draw $\mathrm A'_{1:N}\sim \pi_\mu(\cdot\mid s')$ once and
use the first $M$ samples for the $M$-sample backup. Then for every realization,
\[
\max_{i\in[N]} Q(s',\mathrm A'_i)\ \ge\ \max_{i\in[M]} Q(s',\mathrm A'_i).
\]
Taking expectation over $\mathrm A'_{1:N}$ and $s'$ yields $(\mathcal T_\mu^N Q)(s,\mathrm A)\ge (\mathcal T_\mu^M Q)(s,\mathrm A)$.
Now apply the above pointwise inequality to $Q_\mu^M$. Since $Q_\mu^M=\mathcal T_\mu^M Q_\mu^M$, we have
$Q_\mu^M \le \mathcal T_\mu^N Q_\mu^M$.
Moreover, $\mathcal T_\mu^N$ is order-preserving in $Q$ (max and expectation preserve order), hence
$(\mathcal T_\mu^N)^k Q_\mu^M \ge Q_\mu^M$ for all $k$.
Taking $k\to\infty$ and using Theorem~\ref{thm:contraction} yields
$Q_\mu^N = \lim_{k\to\infty} (\mathcal T_\mu^N)^k Q_\mu^M \ge Q_\mu^M$.

\end{proof}
\begin{theorem}[Limit as $N\to\infty$ (sketch)]
\label{thm:Ninf}
Under standard regularity assumptions (e.g., no ties / continuous proposal density),
the fixed points $\{Q_\mu^{N}\}_{N\ge 1}$ converge pointwise to the optimal value function
whose actions are restricted to the support of $\pi_\mu$, denoted $Q^\star_{\pi_\mu}$:
\[
\lim_{N\to\infty} Q_\mu^{N} = Q^\star_{\pi_\mu}.
\]
\end{theorem}

\begin{proof}[Proof sketch]
For any fixed $Q$ and state $s'$, define
\[
g_N(Q)(s')
=
\mathbb{E}_{\mathrm{A}'_{1:N}\sim\pi_\mu(\cdot\mid s')}
\Big[\max_{i\in[N]} Q(s',\mathrm{A}'_i)\Big].
\]
Then $g_N(Q)(s')$ is non-decreasing in $N$ and converges to the essential supremum of $Q(s',\mathrm{A})$
over $\mathrm{A}\sim\pi_\mu(\cdot\mid s')$.
Consequently, $\mathcal{T}_\mu^N$ converges (pointwise) to the support-restricted optimality operator.
Monotonicity in $N$ (Prop.~\ref{prop:monotonicity}),
together with contraction (Thm.~\ref{thm:contraction}),
yields convergence of the fixed points; see the corresponding EMaQ analysis~\cite{ghasemipour2021emaq}.
\end{proof}


\subsection{\texttt{EGR} Anchoring Envelope and Best-of-$N$ Bound}
\label{app:egr_bestofn_bound}

We formalize how the \texttt{EGR} regression objective induces an \emph{upper-envelope} inequality
over off-demo candidates sampled from $\rho(\cdot\mid s_t)$, and how such an envelope
immediately bounds the Best-of-$N$ maximization used in the chunked TD target.

\paragraph{Setup.}
Fix a data (demonstration) pair $(s_t,\mathrm A_t)\sim \mathcal D$ and an auxiliary candidate distribution
$\rho(\cdot\mid s_t)$.
Recall the \texttt{EGR} reference surface (Eq.~\eqref{eq:ref_surface}):
\begin{equation}
\mathcal Y(s_t,\mathrm A'_t)
:= \operatorname{sg}(y_t) - \beta \|\mathrm A'_t-\mathrm A_t\|_{\mathcal W}^2,
\label{eq:Y_def_appendix}
\end{equation}
where $\beta>0$ and $y_t$ is the chunk-level TD anchor (Eq.~\eqref{eq:td_loss_final}).
Define the pointwise \texttt{EGR} residual
\begin{equation}
\delta_\theta(s_t,\mathrm A'_t)
:= Q_\theta(s_t,\mathrm A'_t) - \mathcal Y(s_t,\mathrm A'_t).
\label{eq:delta_def_appendix}
\end{equation}

\paragraph{From \texttt{EGR} regression to an upper envelope.}
The anchoring loss encourages $\delta_\theta(s_t,\mathrm A'_t)$ to be small on candidates $\mathrm A'_t\sim\rho(\cdot\mid s_t)$.
To make this statement explicit, we assume the residual is uniformly upper-bounded on the candidate set.

\begin{lemma}[Anchoring-induced upper envelope via a residual bound]
\label{lem:egr_upper_envelope}
Assume the residual admits a uniform upper bound on the candidate set:
\begin{equation}
\sup_{\mathrm A'_t\in \mathrm{supp}(\rho(\cdot\mid s_t))}
\delta_\theta(s_t,\mathrm A'_t)
\le \varepsilon,
\label{eq:residual_upper_bound}
\end{equation}
for some $\varepsilon\ge 0$.
Then for all $\mathrm A'_t\in \mathrm{supp}(\rho(\cdot\mid s_t))$ we have the upper-envelope inequality
\begin{equation}
Q_\theta(s_t,\mathrm A'_t)
\le
\operatorname{sg}(y_t)
-\beta\|\mathrm A'_t-\mathrm A_t\|_{\mathcal W}^2
+\varepsilon.
\label{eq:egr_upper_envelope}
\end{equation}
\end{lemma}

\begin{proof}
From the residual definition in~\eqref{eq:delta_def_appendix}, for any $\mathrm A'_t$ we have the identity
\begin{equation}
Q_\theta(s_t,\mathrm A'_t)
=
\mathcal Y(s_t,\mathrm A'_t) + \delta_\theta(s_t,\mathrm A'_t).
\label{eq:Q_decompose}
\end{equation}
For any $\mathrm A'_t\in \mathrm{supp}(\rho(\cdot\mid s_t))$, the bound Eq.~\eqref{eq:residual_upper_bound}
implies $\delta_\theta(s_t,\mathrm A'_t)\le \varepsilon$.
Substituting into Eq.~\eqref{eq:Q_decompose} yields
$
Q_\theta(s_t,\mathrm A'_t)\le \mathcal Y(s_t,\mathrm A'_t)+\varepsilon.
$
Finally, plugging in Eq.~\eqref{eq:Y_def_appendix} gives Eq.~\eqref{eq:egr_upper_envelope}.
\end{proof}

\paragraph{Bounding Best-of-$N$ under the envelope.}
We now show the envelope immediately bounds the Best-of-$N$ maximization.
Recall that the \texttt{EGR} loss is minimized over a broad OOD distribution $\rho(\cdot\mid s_t)$. In our implementation (Appendix~\ref{sec:impl_details}), $\rho$ is constructed as a mixture that explicitly includes samples from the proposal policy $\pi_\mu$. Consequently, $\mathrm{supp}(\pi_\mu) \subseteq \mathrm{supp}(\rho)$. Thus, if the envelope holds for $\rho$, it naturally holds for $\pi_\mu$.
\begin{proposition}[Best-of-$N$ bound under an \texttt{EGR}-style upper envelope]
\label{prop:synergy_appendix}
Let $(s_t,\mathrm A_t)\sim\mathcal D$ and let $\pi_\mu(\cdot\mid s_t)$ be the candidate sampling distribution.
Draw $N$ candidates $\{\mathrm A_t^{\prime i}\}_{i=1}^N\sim \pi_\mu(\cdot\mid s_t)$.
If the envelope Eq.~\eqref{eq:egr_upper_envelope} holds for all $\mathrm A'_t\in \mathrm{supp}(\pi_\mu(\cdot\mid s_t))$ {(inherited from $\rho$)},
then

\begin{equation}
\max_{i\in[N]} Q_\theta(s_t,\mathrm A_t^{\prime i})
\le
\operatorname{sg}(y_t)
-\beta \min_{i\in[N]}\|\mathrm A_t^{\prime i}-\mathrm A_t\|_{\mathcal W}^2
+\varepsilon
\le \operatorname{sg}(y_t)+\varepsilon.
\label{eq:bestofn_bound_appendix}
\end{equation}
\end{proposition}

\begin{proof}
Apply Eq.~\eqref{eq:egr_upper_envelope} to each sampled candidate $\mathrm A_t^{\prime i}$:
\[
Q_\theta(s_t,\mathrm A_t^{\prime i})
\le
\operatorname{sg}(y_t)
-\beta \|\mathrm A_t^{\prime i}-\mathrm A_t\|_{\mathcal W}^2
+\varepsilon.
\]
Taking $\max_{i\in[N]}$ on both sides gives
\[
\max_{i\in[N]} Q_\theta(s_t,\mathrm A_t^{\prime i})
\le
\max_{i\in[N]}
\left(
\operatorname{sg}(y_t)
-\beta \|\mathrm A_t^{\prime i}-\mathrm A_t\|_{\mathcal W}^2
+\varepsilon
\right).
\]
Since $\operatorname{sg}(y_t)$ and $\varepsilon$ do not depend on $i$, and $-\beta(\cdot)$ is decreasing in the distance term,
\[
\max_{i\in[N]}
\left(
\operatorname{sg}(y_t)
-\beta \|\mathrm A_t^{\prime i}-\mathrm A_t\|_{\mathcal W}^2
+\varepsilon
\right)
=
\operatorname{sg}(y_t)
-\beta \min_{i\in[N]}\|\mathrm A_t^{\prime i}-\mathrm A_t\|_{\mathcal W}^2
+\varepsilon,
\]
which proves the first inequality in Eq.~\eqref{eq:bestofn_bound_appendix}.
The second inequality follows since $\min_i \|\mathrm A_t^{\prime i}-\mathrm A_t\|_{\mathcal W}^2\ge 0$.
\end{proof}

\paragraph{Remark.}
The residual bound Eq.~\eqref{eq:residual_upper_bound} is not an additional algorithmic constraint:
it is a compact way to express that the \texttt{EGR} regression error on the candidate set is small.
In practice, $\varepsilon$ can be interpreted as the worst-case anchoring fit error over candidates drawn from $\pi_\mu(\cdot\mid s_t)$.

\section{Visualization Methodology and Detailed Analysis}
\label{app:viz_analysis}

Figure~\ref{fig:value_map_comparison} visualizes the landscape of estimated state-action-chunk values on the LIBERO-Object benchmark. To interpret the high-dimensional action chunks (dimension $h \times d_{\text{action}}$), we use Principal Component Analysis (PCA) to project proposal candidates and ground-truth (GT) demonstrations onto a 2D plane, normalizing estimated $Q$-values to $[-1, 1]$ for comparison. The color gradient---from yellow (high value) to dark blue (low value)---reveals clear differences in landscape topology between the two approaches. The top row shows that CQL-style regularization results in a collapsed value landscape: it indiscriminately suppresses all proposal candidates (white dots) to a uniformly low value regardless of their quality, failing to differentiate ``near-miss'' proposals from failures. In contrast, the bottom row demonstrates that VGAS (with \texttt{EGR}) yields a discriminative value landscape. Instead of a binary plateau, the value signal exhibits a graded geometric preference, decaying smoothly as candidates deviate from the expert trajectory. This confirms that \texttt{EGR} acts as a structural inductive bias,  enabling the critic to meaningfully rank candidates based on their physical proximity to the optimal solution.

We further observe distinct spatial and temporal characteristics in the learned landscapes. Spatially, there is a visible gap between the proposal distribution  and the GT (red dot); we attribute this to the mean-seeking bias of the SFT policy, which tends to generate smoothed trajectories compared to the sharper, high-frequency control signals of human experts. Crucially, VGAS maintains a valid gradient across this gap, guiding selection toward the expert mode despite the distribution shift. Temporally, the value maps exhibit dynamic coherence: adjacent timesteps ($t=0$ and $t=1$) share similar topologies, whereas distant steps ($t=50$ and $t=120$) show significant structural differences. This indicates that the VGAS does not merely memorize static geometric relations but adaptively adjusts its estimation according to evolving real-world dynamics, providing state-aware guidance throughout the entire task horizon.

\section{Implementation Details}
\label{sec:impl_details}

\subsection{Experimental Setup}
\label{sec:exp_setup}

\paragraph{(i) Reward function.}
Consistent with real-world scenarios where explicit rewards are scarce, the LIBERO dataset contains no reward annotations. We therefore construct a sparse binary reward by labeling the last $h$ steps of each demonstration rollout as successful. Following~\cite{nakamoto2025steering}, we set $h=3$. Additionally, we use shifted rewards $\{-1,1\}$ instead of $\{0,1\}$, which we found to yield more stable learning in practice.
\paragraph{(ii) Training protocol.}
We first perform supervised fine-tuning (SFT) of the VLA model using 5-shot expert demonstrations per task, randomly sampled from the LIBERO dataset. We then train a critic using different variants of offline RL (ORL) objectives on the resulting offline data.
\paragraph{(iii) OOD candidate generation.}
Following the conservative offline RL intuition (e.g., CQL) that penalizes actions with low support under the offline dataset, we construct OOD candidates by sampling from $\rho(\cdot\mid s_t)$, instantiated as a mixture of: (i) proposals from the frozen VLA policy $\pi_\mu(\cdot\mid s_t)$; (ii) Gaussian perturbations of the demonstration (ground-truth) chunk $\mathrm A_t$ (GT$+$noise); (iii) prefix-truncated variants of $\mathrm A_t$ (early-terminated chunks); and (iv) linear interpolations between the demonstration chunk and a policy proposal, i.e., $\tilde{\mathrm A}_t
= \alpha\, \mathrm A_t + (1-\alpha)\, \hat{\mathrm A}_t .$ and $\alpha\in[0,1]$.
\begin{figure*}[t]
  \centering
  \includegraphics[width=\textwidth]{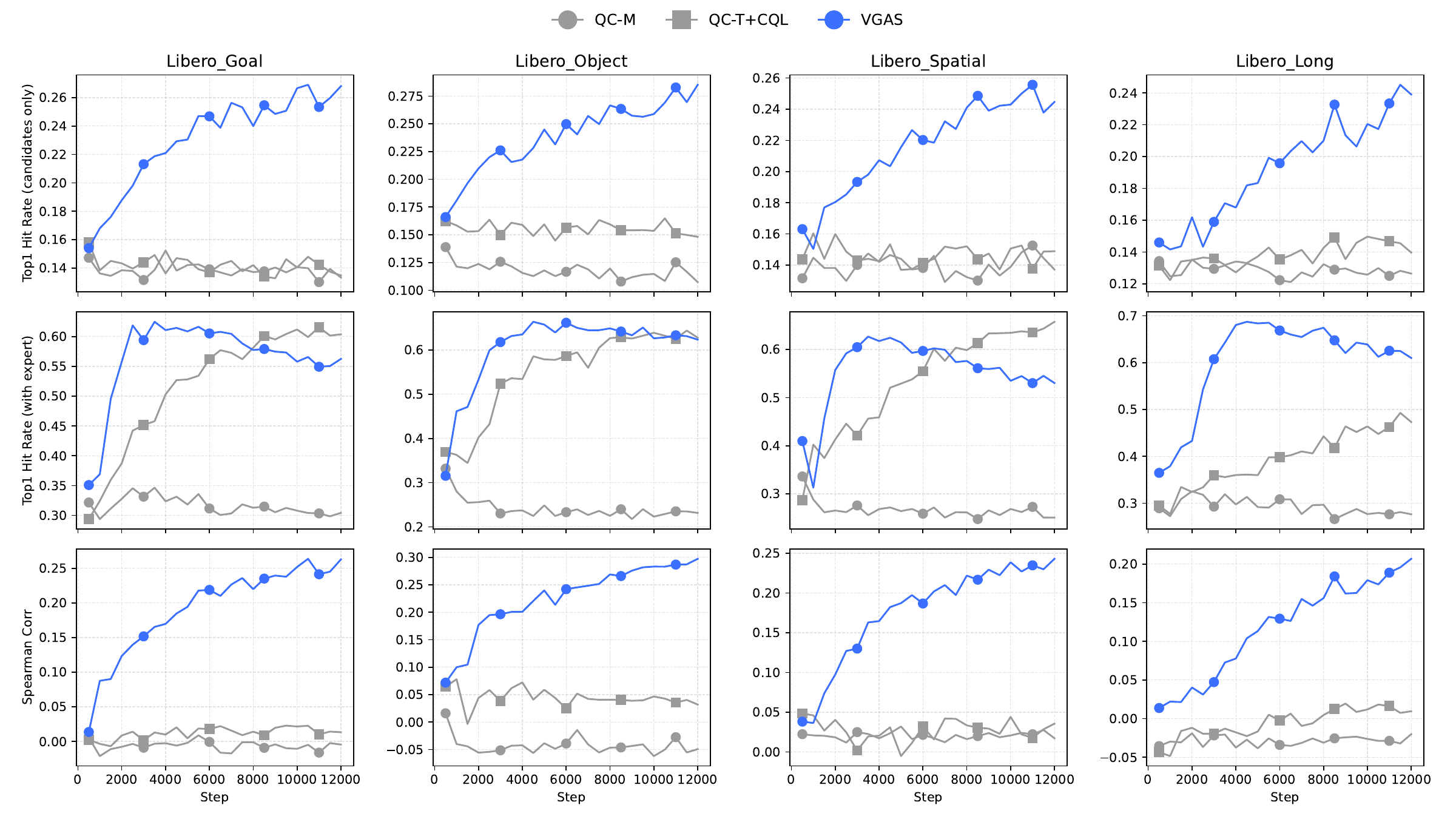}
  \caption{Offline Ranking Evaluation on Held-out Data.}
  \label{fig:top1_and_spearman}
  \vspace{-0.2cm}
\end{figure*}
\paragraph{(iv) Backbone Architecture Details.} We adopt SmolVLA-0.5B~\cite{shukor2025smolvla} as the underlying policy backbone. Built upon the SmolVLM2 architecture, it employs a SigLIP visual encoder to process high-dimensional observations. These visual tokens are fused with language instructions and proprioceptive states within a decoder-only Transformer. The policy head utilizes a flow-matching objective to generate high-dimensional action chunks in parallel, which serves as the proposal distribution for our VGAS framework. 

\subsection{Baseline Implementation Details}
\label{sec:baseline}

\paragraph{Base Policy Training.}
To evaluate performance in a realistic data-scarce regime, we adopt a standardized two-stage training protocol for all methods. First, we obtain the base proposal policy $\pi_\mu$ by fine-tuning the SmolVLA backbone using only 5 expert demonstrations per task, randomly sampled from the LIBERO benchmark\footnote{\url{https://huggingface.co/datasets/HuggingFaceVLA/libero}}. This fine-tuning is conducted via the official LeRobot~\cite{cadene2024lerobot} repository\footnote{\url{https://github.com/huggingface/lerobot}}. In the second stage, we freeze this few-shot adapted policy to serve as a static proposal generator. Crucially, the offline dataset $\mathcal{D}$ used to train the critic (for both VGAS and baselines) is constructed exclusively from these same 5-shot demonstrations, ensuring that value learning operates under the same strict data-scarce constraints.

\paragraph{Baseline Configurations.}
Unless stated otherwise, all methods keep $\pi_\mu$ fixed and differ only in the critic objective and architecture.
\begin{itemize}
    \item \textbf{BC-Only:} This baseline directly executes the base policy $\pi_\mu$ obtained in the first stage without any test-time selection.
    
    \item \textbf{QC-M (Q-Chunking-MLP):} Q-Chunking~\cite{li2025reinforcement} proposes two variants: QC (Best-of-$N$ backup) and QC-FQL (Flow Matching distillation). We adopt the QC variant for direct comparison. Following the official implementation, we use an MLP critic with hidden dimensions $[512,512,512,512]$.
    For state encoding, we pool the VLM tokens ($I_t, L_t$ and $p_t$) from the backbone into a fixed representation (dim=960). The action chunk ($h \times d_a$) is flattened and projected via an MLP to match the state dimension. Finally, the state and action chunk representations are concatenated and fed into the critic.
    
    \item \textbf{QC-M+CQL:} We augment QC-M by adding the CQL regularization term. We use the default coefficient $\alpha=5.0$, which performs comparably to alternative choices (e.g., $\alpha=2.0$) in our validation.
    The negative (OOD) action pool is sampled from the proposal-centered distribution $\pi_\mu(\cdot\mid s_t)$, ensuring alignment with our method's $\rho(\cdot\mid s_t)$ construction for fair comparison

    \item \textbf{QC-T+CQL:} We retain the QC-M+CQL objective but replace the MLP backbone with our Transformer-based \textrm{Q-Chunk-Former}. This baseline isolates the effect of critic architecture under identical conservative constraints.
\end{itemize}

\begin{table}[t]
\centering
\caption{Hyperparameters for VGAS.}
\label{tab:hyperparams_vgas}
\begin{tabular}{@{}l l@{}}
\toprule
\textbf{Hyperparameter} & \textbf{Value}\\
\midrule
Learning rate & 0.0001\\
Optimizer & Adam \cite{kingma2014adam} \\
 Warm up step&1000\\
Gradient steps & 12000 (default), 20000 (Libero-Goal, Libero-long)\\
 Gradient-clip&10.0\\
Minibatch size & 32\\
Q-chunk-Former layers& 2\\
Q-chunk-Former hidden dimensions& 960\\
State-Action-Fusion Mlp dimensions& 960\\
Q chunk len& 32(default), 50 (Libero Long)\\
 N-action-step&20\\
Vlaue Head MLP dimensions& [512, 512]\\
Vlaue head Nonlinearity& GELU \cite{hendrycks2016gaussian}\\
Target network smoothing coefficient& 0.005\\
Discount factor $\gamma$& 0.98 (default), 0.99 (Libero-Long)\\
Clipped double Q-learning& True\\
Q aggregation (twin critics) & Min\\
\bottomrule

\end{tabular}
\end{table}

\begin{table}[t]
\centering
\caption{Hyperparameters for Different Task suit.}
\label{tab:hyperparams_task_specific}
\begin{tabular}{lcccc}
\toprule
Task suit & $\lambda$ & $\beta$ & $\eta$ & $\mathrm w$ \\
\midrule
Libero-Spatial & 5.0 & 5.0 & 1.0 & {5 5 5 1 1 1 1} \\
Libero-Goal    & 5.0 & 5.0 & 1.0 & {5 5 5 1 1 1 1} \\
Libero-Object  & 2.0 & 2.0 & 1.0 & {5 5 5 1 1 1 1} \\
Libero-Long    & 5.0 & 5.0 & 1.0 & {5 5 5 1 1 1 1} \\
\bottomrule
\end{tabular}
\end{table}

\paragraph{Shared rollout settings.}
Unless otherwise specified, all methods use the same chunk length $h=32$ (or $h=50$ for LIBERO-Long) to estimate $Q(s_t,\mathrm{A}_t)$, where $\mathrm{A}_t := (a_t, a_{t+1}, \dots, a_{t+h-1})$ denotes an $h$-step action chunk. During execution, we apply only the first $n_{\text{exec}}=20$ action steps of each selected chunk, i.e., $(a_t, a_{t+1}, \dots, a_{t+n_{\text{exec}}-1})$ with $n_{\text{exec}}<h$.

\paragraph{VGAS.}
We report the complete list of hyperparameters in Table~\ref{tab:hyperparams_vgas}, and task-specific overrides in Table~\ref{tab:hyperparams_task_specific}. Our \textrm{Q-Chunk-Former} is initialized from the first two layers of the SmolVLM backbone. We directly reuse the multimodal features extracted by the frozen SmolVLM (i.e., the output of the SmolVLA encoder) as the vision--language input to \textrm{Q-Chunk-Former}. In our notation, the \emph{Q-chunk length} $h$ denotes the length of an action chunk, while \emph{$N$-action-step} indicates that we execute only the first $N$ primitive actions within each predicted chunk at rollout. We use clipped double Q-learning: the \texttt{TD} target is computed using the minimum of the two target critics. 

\subsection{Weighted metric implementatipn Details}
\label{weight_detail}
In the Explicit Geometric Regularization (\texttt{EGR}) objective (Eq.~\ref{eq:ref_surface} and Eq.~\ref{eq:egr_pair_dist}), the term $\lVert \hat{\mathrm A}_t - \mathrm A_t \rVert_{\mathcal W}^2$ denotes a \emph{weighted squared Euclidean distance}, averaged over the valid horizon of an action chunk. This metric prioritizes critical action dimensions (e.g., end-effector translation) and supports variable-length chunks via a padding mask.

Let $\mathrm A \in \mathbb{R}^{h \times d_{a}}$ be the ground-truth action chunk and $\hat{\mathrm A}\in \mathbb{R}^{h \times d_{a}}$ be a candidate chunk, where $h$ is the chunk size and $d_{a}$ is the action dimension ((for LIBERO, typically $d_a=7$: 3D end-effector position, 3D end-effector orientation (a 3-parameter representation), and 1D gripper control)
). Let $\mathrm w \in \mathbb{R}^{d_a}$ be a nonnegative weight vector. We define the masked, weighted squared distance as
\begin{equation}
\label{eq:weighted_geom_dist}
\lVert \hat{\mathrm A} - \mathrm A \rVert_{\mathcal W}^2
=
\frac{1}{\sum_{k=1}^{H} m_k}
\sum_{k=1}^{H} m_k
\sum_{j=1}^{d_a} w_j \left(\hat{a}_{k,j} - a_{k,j}\right)^2,
\end{equation}
where:
\begin{itemize}
  \item $m_k \in \{0,1\}$ is a binary mask indicating whether time step $k$ is valid (i.e., non-padding),
  \item $\hat{a}_{k,j}$ and $a_{k,j}$ denote the $j$-th action dimension at step $k$ for $\hat{\mathrm A}$ and $\mathrm A$, respectively,
  \item $\mathrm w = [w_1,\dots,w_{d_a}]$ specifies the relative importance of each control dimension.
\end{itemize}

\paragraph{Weight configuration.}
As shown in Table~\ref{tab:hyperparams_task_specific}, we set
$\mathrm w={[5, 5, 5, 1, 1, 1, 1]}$,
assigning weight $5.0$ to translational components $(x,y,z)$ and weight $1.0$ to rotation and gripper states. Empirically, penalizing position errors more heavily encourages the critic to emphasize trajectory precision, which is crucial for manipulation success.


\section{Additional Experimental Results and Analysis}
\label{app:additional_exp}
\subsection{Training Dynamics and Critic Ranking Resolution}
\label{app:training dynamics}

To verify the critic’s generalization ability, we evaluate it offline on a held-out set of 45 unseen expert episodes per task. For each observation, we sample $N=8$ proposal candidates and report the Top-1 Hit Rate (the probability of selecting the proposal that is geometrically closest to the expert) and the Spearman correlation. Fig.~\ref{fig:top1_and_spearman} reveals three key findings. First, VGAS consistently assigns higher scores to candidates that remain close to the expert, which reduces overestimation on geometrically divergent OOD actions and supports safer selection on unseen states where the policy may drift. Second, the high hit rate indicates that VGAS builds a discriminative value landscape within the proposal set. In contrast, CQL tends to flatten local value differences, weakening ranking resolution. VGAS preserves a clear ordering among “near-miss” candidates, making superior candidates easier to separate from inferior ones. Finally, comparing the “With Expert” setting (middle row) and the “Candidates Only” setting (top row) highlights an important difference. CQL can assign high scores to the expert action when it is included, but its local ranking becomes less reliable when only proposal candidates are available. VGAS performs well in both settings. Although geometric proximity is not a perfect proxy for task success, it provides a practical signal in the few-shot imitation regime, where unconstrained maximization is brittle.

\begin{figure}[H]
  \centering
  \includegraphics[width=0.7\textwidth]{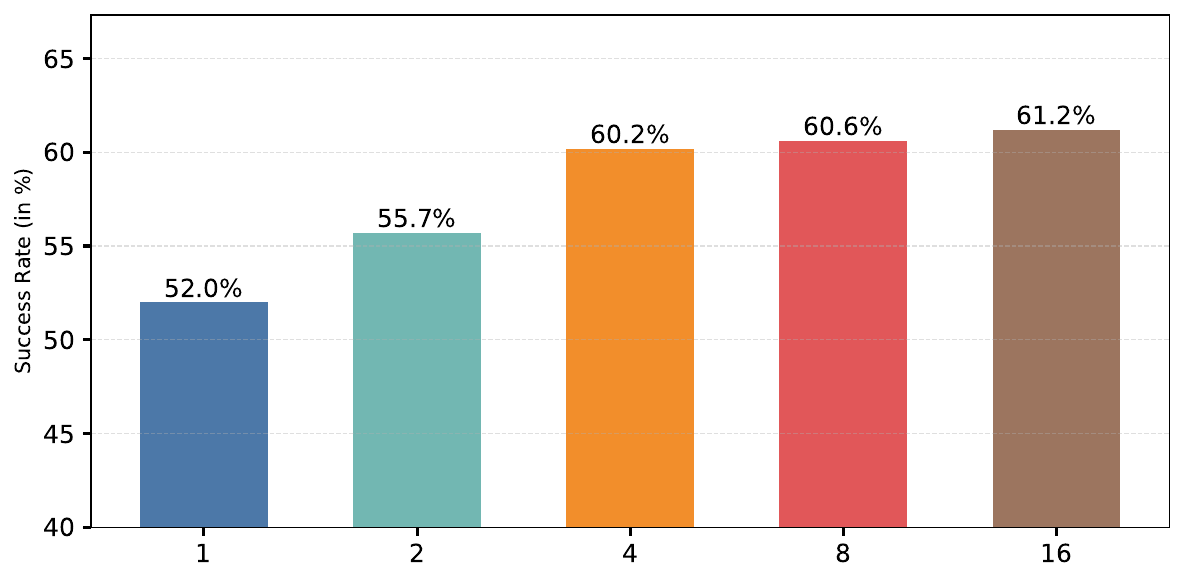}
  \caption{Impact of Inference Budget $N$. Evaluation on LIBERO-Goal showing monotonic improvement from the baseline ($N=1$) to saturation around $N=8$.}
  \label{fig:study_n}
  \vspace{-0.2cm}
\end{figure}
\FloatBarrier

\subsection{Sensitivity to Inference Budget }

We study how the number of sampled proposals $N$ affects VGAS, using the LIBERO-Goal suite as a representative case. As shown in Fig.~\ref{fig:study_n}, the success rate increases monotonically with the inference budget. Starting from the base policy ($N{=}1$, $52.0\%$), applying Best-of-$N$ selection yields substantial gains, reaching $60.2\%$ with only $N{=}4$. This rapid improvement indicates that the proposal policy often generates high-quality ``near-miss'' candidates that are not selected under greedy sampling but can be reliably retrieved by our geometrically regularized critic. Beyond $N{=}8$ ($60.6\%$), the gains diminish, with only a marginal improvement at $N{=}16$ ($61.2\%$). We therefore use $N{=}8$ as the default, balancing performance against inference latency in the main experiments.

\subsection{Additional Policy Baselines}
\label{app:vla_baselines}

\begin{figure}[t]
  \centering
  \includegraphics[width=0.7\textwidth]{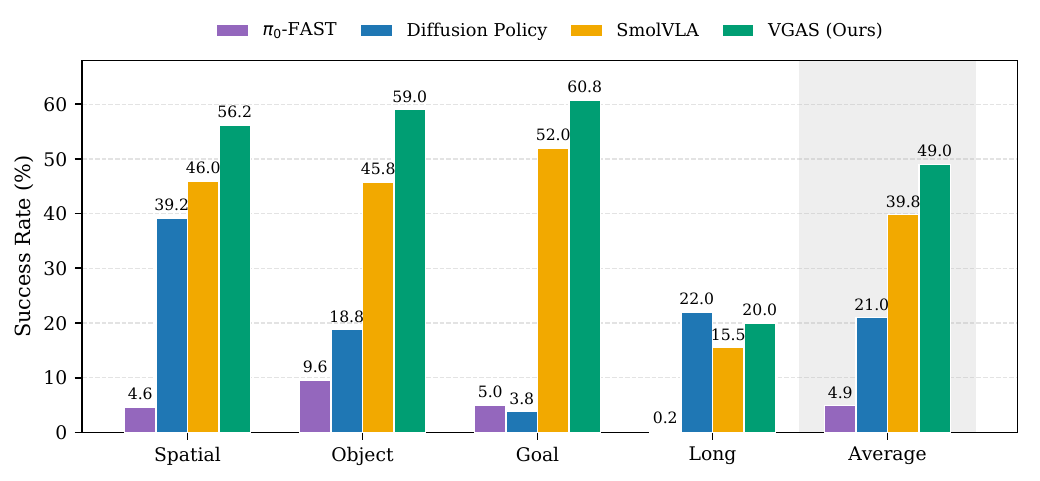}
  \caption{Additional VLA baselines under the 5-shot LIBERO setting. VGAS improves over the SmolVLA base policy and achieves the best average success rate among the compared methods.}
  \label{fig:additional_vla_baselines}
\end{figure}

Beyond the SmolVLA backbone used in our main controlled experiments, we also compare VGAS with two additional policies under the same 5-shot LIBERO setting: $\pi_0$-FAST and Diffusion Policy. $\pi_0$-FAST~\cite{pertsch2025fast} is an autoregressive action-token VLA policy that represents a continuous action chunk as a discrete token sequence and decodes the generated tokens back into continuous control. Diffusion Policy~\cite{chi2025diffusion} directly models continuous action sequences through iterative denoising. For both baselines, we follow the official implementations and use their default hyperparameter settings unless otherwise specified.

As shown in Fig.~\ref{fig:additional_vla_baselines}, VGAS achieves the best average success rate and consistently improves over the SmolVLA base policy, showing that value-guided selection provides gains beyond the underlying VLA proposal generator. $\pi_0$-FAST performs poorly in the 5-shot setting, likely due to the sensitivity of autoregressive action-token generation to token-level errors under scarce supervision. Diffusion Policy performs reasonably on LIBERO-Spatial and LIBERO-Long, but remains below VGAS on average. These results further support the effectiveness of value-guided action-chunk selection in few-shot adaptation.

\subsection{Scalability Across Demonstration Budgets}
\label{app:shot_scaling}

\begin{figure}[t]
  \centering
  \includegraphics[width=0.8\textwidth]{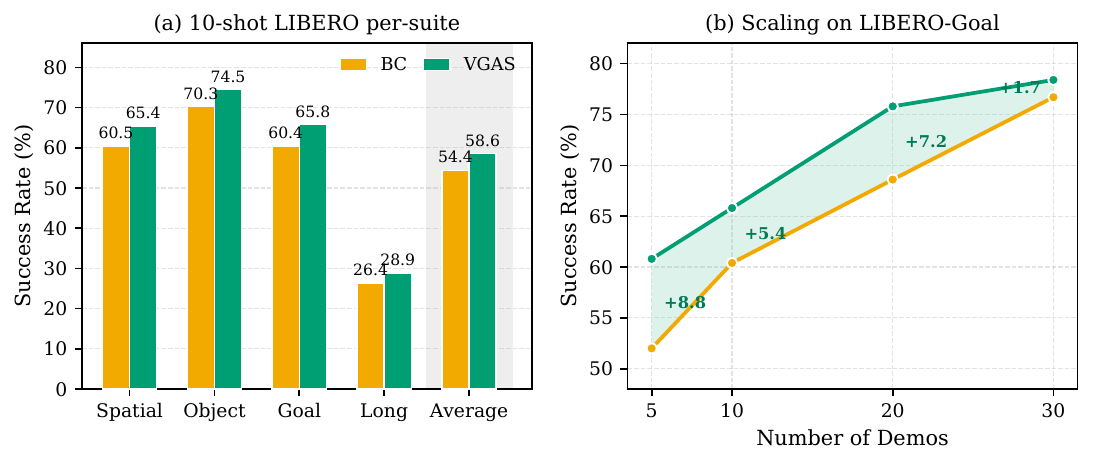}
  \caption{Scalability across demonstration budgets. 
  (a) Per-suite results under the 10-shot LIBERO setting. 
  (b) Scaling trend on LIBERO-Goal with 5, 10, 20, and 30 demonstrations. 
  VGAS consistently improves over the BC baseline, with larger gains in lower-data regimes and gradually diminishing margins as more demonstrations are provided.}
  \label{fig:shot_scaling}
  \vspace{-0.2cm}
\end{figure}

We further evaluate VGAS under different demonstration budgets to assess its scalability beyond the default 5-shot setting. As shown in Fig.~\ref{fig:shot_scaling}(a), VGAS consistently improves over the BC baseline across all four LIBERO suites in the 10-shot setting. Fig.~\ref{fig:shot_scaling}(b) further shows that VGAS maintains a positive improvement over BC on LIBERO-Goal across 5, 10, 20, and 30 demonstrations.

The gains are generally larger in low-data regimes, where the base policy remains semantically plausible but geometrically under-constrained. As the number of demonstrations increases, the BC baseline becomes more accurate, reducing the potential improvement from value-guided selection. This trend is consistent with our motivation. VGAS is particularly effective in few-shot regimes where the base policy has not yet learned sufficiently precise geometric control, but still retains enough local recall to generate candidates close to successful behavior. In such cases, value-guided selection can identify a better action chunk from the candidate set.

\subsection{Hyperparameter Sensitivity}
\label{app:hyper_sensitivity}

We evaluate the sensitivity of VGAS to key critic hyperparameters on LIBERO-Goal under the 5-shot setting. We conduct a controlled one-at-a-time sweep over $\lambda$, $\beta$, and $\eta$, fixing the remaining coefficients to a common reference setting. All runs use Best-of-$N$ inference with $N=8$ and results are averaged over five evaluation seeds. As shown in Fig.~\ref{fig:param_sensitivity}, VGAS remains stable across all three sweeps. The success rate varies only mildly, with ranges of $1.80$, $0.72$, and $2.32$ points for $\lambda$, $\beta$, and $\eta$, respectively, and consistently stays above the BC baseline.

\begin{figure}[H]
  \centering
  \includegraphics[width=1.0\textwidth]{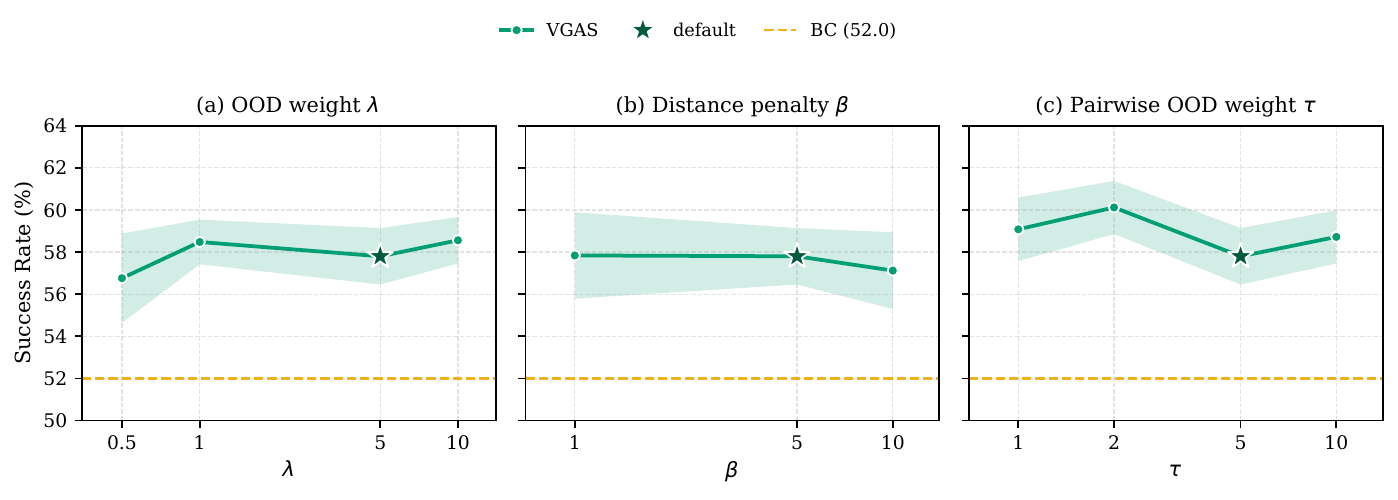}
  \caption{Hyperparameter sensitivity on LIBERO-Goal under the 5-shot setting. 
We sweep one coefficient at a time: 
(a) $\lambda \in \{0.5,1,5,10\}$ with $\beta=5,\eta=5$; 
(b) $\beta \in \{1,5,10\}$ with $\lambda=5,\eta=5$; 
and (c) $\eta \in \{1,2,5,10\}$ with $\lambda=5,\beta=5$. 
The green star marks the common reference setting used in this sensitivity study, and the orange dashed line denotes the BC baseline.}
  \label{fig:param_sensitivity}
  \vspace{-0.2cm}
\end{figure}

\end{document}